\documentclass[10pt]{article} 
\def\comments{0}

\usepackage{preamble}
\newcommand{\users}{\mathify{t}}
\newcommand{\samples}{\mathify{n}}
\newcommand{\samplesp}{\mathify{n}_{\text{pers}}}
\newcommand{\datadim}{d}
\newcommand{\sampledom}{\cZ}

\newcommand{\outdom}{\cO}
\newcommand{\bbdom}{\cB}

\newcommand{\dist}{P}
\newcommand{\distfam}{\cP}
\newcommand{\hyp}{g}
\newcommand{\rep}{h}
\newcommand{\metadist}{\cQ}
\newcommand{\alg}{\cM}
\newcommand{\err}{\ell}
\newcommand{\Det}{\text{Det}}
\DeclarePairedDelimiterX{\infdivx}[2]{(}{)}{%
  #1\;\delimsize\|\;#2%
}
\newcommand{\renyi}{\mathrm{D}_\alpha\infdivx}

\title{Privacy in Metalearning and Multitask Learning: \\ Modeling and Separations}
\author{
    Maryam Aliakbarpour\thanks{Department of Computer Science, Rice University. Supported by NSF awards CNS-2120667, CNS-2120603, CCF-1934846, and the Hariri Institute for Computing at Boston University. Part of this work was conducted while M.A.\ was affiliated with Boston University and Northeastern University, and while serving as a research fellow at the Simons Institute for the Theory of Computing.
    } \and
    Konstantina Bairaktari\thanks{Khoury College of Computer Sciences, Northeastern University.  Supported by NSF awards CNS-2232692 and CCF-2311649.} \and
     \and
    Adam Smith\thanks{Department of Computer Science, Boston University. Work at BU supported by  NSF awards CCF-1763786 and CNS-2120667, Faculty Awards from Google and Apple, and Cooperative Agreement CB16ADR0160001 with the Census Bureau. The views expressed in this paper are those of the authors and not those of the U.S. Census
    Bureau or any other sponsor.} \and
    Marika Swanberg\thanks{Google. This work was done while M.S.\ was affiliated with Boston University, supported by  NSF awards CCF-1763786 and CNS-2120667.}
    \and
    Jonathan Ullman\thanks{Khoury College of Computer Sciences, Northeastern University.  Supported by NSF awards CNS-2232692 and CNS-2247484.}
}
\date{}

\begin{document}

\maketitle

\begin{abstract}
Model personalization allows a set of individuals, each facing a different learning task, to train models that are more accurate for each person than those they could develop individually. The goals of personalization are captured in a variety of formal frameworks, such as multitask learning and metalearning. Combining data for model personalization poses risks for privacy because the output of an individual's model can depend on the data of other individuals.  In this work we undertake a systematic study of differentially private personalized learning. Our first main contribution is to construct a taxonomy of formal frameworks for private personalized learning. This taxonomy captures different formal frameworks for learning as well as different threat models for the attacker. Our second main contribution is to prove separations between the personalized learning problems corresponding to different choices.  In particular, we prove a novel separation between private multitask learning 
and private metalearning.
\end{abstract}

\vfill\newpage
\tableofcontents
\vfill\newpage





\section{Introduction}
\label{sec:intro}

Model personalization allows a set of individuals, each facing a different learning task, to train models that are more accurate for each person than those they could develop individually. For example, consider a set of people, each of whom holds a relatively small dataset of photographs labeled with the names of their loved ones that appear in each picture.  Each person would like to build a classifier that labels future pictures with the names of people in the picture, but training such an image classifier would take more data than any individual person has.  Even though the tasks they want to carry out are different---their photos have different subjects---those tasks share a lot of common structure.  By pooling their data, a large group of people could learn the shared components of a good set of classifiers.  Each individual could then train the subject-specific components on their own, requiring only a few examples for each subject.  Other applications of personalization include next-word prediction on a mobile keyboard, speech recognition, and recommendation systems.

The goals of personalization are captured in a variety of formal frameworks, such as multitask learning and metalearning. Roughly, in \emph{multitask learning} we are given data for $t$ tasks and wish to find $t$ different, but related, models $g_1,\ldots, g_t$ that each perform well on one task. In \emph{metalearning}, we aim to find a common \emph{representation}, which is a summary (e.g.\ an embedding of the data into a lower-dimensional space) that can be adapted to new test tasks that are similar to the training tasks, and where similarity is formalized through some task distribution.

Although it offers many benefits, model personalization poses new risks for privacy because each individual's model depends on the data of other individuals.  Thus, we undertake a systematic study of \emph{differentially private (DP) personalized learning}.  Informally, DP personalized learning requires that the components of the models made visible to \emph{people other than you} do not reveal too much information about \emph{your data}.  As in the motivating example above, we think of your data as being synonymous with the training data for a single learning task.

\paragraph{Frameworks for private personalization.}
Our first contribution is a taxonomy of different formal frameworks for private personalized learning, which vary according to the learning objective, the privacy requirements, and the structure of the output.

Our main interest in investigating these models is sample complexity: How many individuals/tasks $\users$ need to contribute data to achieve a given learning objective, and how many samples $\samples$ do we need from each individual/task?  We focus on small $\samples$, so no one individual data can learn well on their own. 

Because we focus on sample complexity (as opposed to computation or communication), we consider a centralized curator that can collect and process the data of all parties. Such a curator can typically be simulated by a secure multiparty computation without the need for an actual trusted party. Specifically, we consider the following output structures (omitting citations---see the more detailed discussion and related work sections below)\adam{Adding excuses.}: 
\begin{itemize}
    \item[] \emph{Billboard}: The outputs of the curator are visible to all individuals (``written on a billboard''), but may be adapted locally by each individual for their use. This is the model used in the overwhelming majority of papers on private personalized learning.
    \item[] \emph{Separate Outputs}: The curator sends a separate output to each individual, that is only visible to them. 
\end{itemize}

Each output structure leads to one or more corresponding privacy requirements.  In the billboard structure, it is clear that the DP condition has to apply to everything the curator outputs.  When individuals receive separate outputs, we consider two types of attackers.  One type sees the outputs of \emph{all but one} individual and we want this attacker to be unable to make inferences about the data of the remaining individual (\emph{joint DP}).  In the other ($1$-out-of-$t$ DP), we require privacy only for an individual that sees \emph{one} output.

\mypar{Relating the frameworks.}
Our second contribution is a set of relationships among and separations between the different combinations of learning objectives and privacy requirements. 

We introduce two basic problems, one involving estimation and one involving classification, and show a hierarchy of frameworks for both problems:

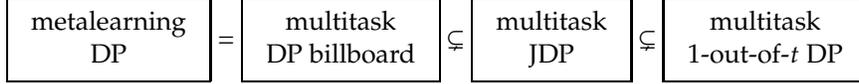
\begin{center}
\fbox{\begin{tabular}{c}
     metalearning \\
     DP\\
\end{tabular}}
$=$
\fbox{\begin{tabular}{c}
     multitask \\
     DP billboard\\ 
\end{tabular}}
$\subsetneq$
\fbox{\begin{tabular}{c}
     multitask \\
     JDP\\ 
\end{tabular}}
$\subsetneq$
\fbox{\begin{tabular}{c}
     multitask \\
     1-out-of-$t$ DP\\ 
\end{tabular}}
\end{center}

\noindent Our results have a number of technical and conceptual implications for private personalized learning:
\begin{enumerate}
    \item Multitask learning with billboard DP implies DP metalearning and vice-versa.  
    \item For some tasks, multitask learning with a general JDP algorithm can require much less data than for DP billboard algorithms.  Thus, while most JDP learning algorithms to date use a billboard, there is potential to significantly reduce the sample complexity by exploring different algorithmic paradigms.  As a consequence, there is a separation between DP multitask and DP metalearning, even though the two objectives are essentially equivalent in the absence of privacy constraints.  
    \item There is a separation between private multitask learning with JDP and with $1$-out-of-$\users$ DP.  Thus, private personalized learning require much less data if we are in a setting where a realistic adversary can only see the outputs or behavior of a small number of individuals.
\end{enumerate}




\subsection{Problem Formulation and Results}
\label{sec:contributions}

In this section we describe the key concepts in our taxonomy and then describe our technical results.

\mypar{Learning Objectives.}
We consider two learning objectives---\emph{multitask learning} and \emph{metalearning}.  In both, data from $\users$ people is pooled together for learning. In multitask learning we simply want to learn a separate model for each of these tasks, but in metalearning, we want to extract some kind of representation that can later be specialized to another unseen task.  Note that metalearning is intuitively the harder objective because the representation can also be specialized to the original tasks.

In \emph{multitask learning}, we have $t$ tasks, each of which is modeled by a distribution $\dist_i$, with training data for the task sampled i.i.d.\ from this distribution.  The algorithm takes in training datasets $S_i$ for each task and returns hypotheses $\hyp_i$ for each task. Our goal is for the $\hyp_i$ to have low error on their respective tasks, as measured by some loss function $\err(\dist_i, \hyp_i)$, averaged  over the $\users$ tasks.
\begin{defn}[Multitask Learning]
    \label{def:multitask_learning} Let $\distfam$ be a set of $\users$-tuples of distributions and $\err$ be a loss function.  An algorithm $\alg$ \emph{multitask learns $\distfam$ with error $\alpha$ with $\users$ tasks and $\samples$ samples per task} if, when given datasets $S_1 \sim \dist_1^n,\dots,S_\users \sim \dist_\users^n$ for any $(\dist_1,\dots,\dist_\users) \in \distfam$, it returns hypotheses $\hyp_1,\dots,\hyp_\users$ such that 
    \[
        \mathbb{E}\left[\frac{1}{\users} \sum_{i \in [\users]} \err(\dist_i,\hyp_i)\right] \leq \alpha
    \]
    where the expectation is taken over the datasets and the randomness of $\alg$.
\end{defn}
Note that the definition doesn't make assumptions about the individual tasks or relationships between them, but some assumptions will be necessary for any non-trivial multitask learning.

The other learning objective we consider is \emph{metalearning}, which models learning a \emph{representation} that can be used for some unseen task.  Here, we have a collection of $\users$ \emph{training tasks} and a separate \emph{test task} drawn from a task distribution.  The first algorithm takes training datasets $S_i$ for each of the training tasks and outputs a representation\footnote{We do not constrain the form of $h$ in any way, and use ``representation'' to mean any summary of the training data that can be specialized to future tasks.  In fact, if privacy were not a concern, the representation could consist of raw training data.} $h$.  The second algorithm then takes $h$ and training data for the unseen test task and returns a hypothesis $\hyp$.  Our goal is for the final hypothesis $\hyp$ to have low error on the test task.
\begin{defn}[Metalearning]
    \label{def:metalearning}
    Let $\metadist$ be a distribution over $(\users+1)$-tuples of distributions and $\err$ be a loss function.  A pair of algorithms $(\alg_{\text{meta}}, \alg_{\text{pers}})$ \emph{metalearns $\metadist$ with error $\alpha$ using $\users$ training tasks, $\samples$ samples per training task, and a test task with $\samplesp$ personalization samples} if the following holds: 
    Let $(\dist_1,\dots,\dist_{\users+1})$ be a tuple of tasks drawn from $\metadist$.  Let $S_i \sim \dist_i^\samples$ for $i \in [\users]$ and give $S_1,\dots,S_\users$ to $\alg_{\text{meta}}$ to obtain a representation $h$.  Now give $h$ and $S_{\users+1} \sim \dist_{\users+1}^{\samplesp}$ to $\alg_{\text{pers}}$ to obtain a hypothesis $\hyp_{\users+1}$.  Then
    \[
      \mathbb{E}[\err(\dist_{\users+1},\hyp_{\users+1})] \leq \alpha  
    \]
    where the expectation is taking over the choice of tasks, the training datasets, the test dataset, and the randomness of the algorithms.
\end{defn}


\mypar{Privacy Requirements for Personalized Models.}
We consider a range of privacy requirements  based on \emph{differential privacy}, which we summarize here informally (see \S\ref{sec:privacy-defs} for precise definitions).  Roughly, an algorithm $\alg$ is differentially private if the distribution of its output is insensitive to changing one individual's data.  In our context, this means that if we have $S = (S_1,\dots,S_\users)$, where each $S_i$ is all the training data of a specific individual/task, and $S'$ which differs on the data for one individual, then $\alg(S)$ and $\alg(S')$ have nearly the same distribution.

Our  privacy requirements vary according to what we assume to be visible to an attacker. 
Some of these requirements are specific to one of the two learning objectives above, or assume particular structural constraints on the algorithms.  

The privacy requirement for metalearning is easiest to describe, because we assume that the representation $h$ is published, and thus require that $h$ does not reveal too much about any of the individuals' data.  Formally, this means simply that algorithm $\alg_{\text{meta}}$ is differentially private.

Multitask learning offers a richer space of possible privacy requirements, because the learning framework allows for each person $i$ to receive a different output $\hyp_i$.  In this context, the natural privacy requirement is \emph{joint differential privacy (JDP)}.  Here, we imagine that each individual $i$ is given only their own model $\hyp_i$.  For any non-trivial learning, $\hyp_i$ must depend on individual $i$'s dataset $S_i$.  However, we do not want an attacker who can observe other players' outputs to learn about individual $i$'s data set.
%
So, we require that for every individual $i$, the collection of $\users - 1$ models $\hyp_{-i}$ is differentially private as a function of the dataset $S_i$ belonging to individual $i$.

The above definition requires that $S_i$ is protected even if every other individual colludes and combines their models.  We can also consider a relaxation of this definition called \emph{1-out-of-$\users$ differential privacy} where we do not allow individuals to collude.  Here, we require that for every individual $i$, and for every other individual $i' \neq i$, the model $\hyp_{i'}$ given to individual $i'$ is differentially private as a function of the dataset $S_i$ belonging to individual $i$.

The majority of algorithms that satisfy joint differential privacy have a particular form, called a \emph{billboard algorithm}. This concept is general, but we will describe it in the context of multitask learning.  In a billboard algorithm, we decompose $\alg$ into two phases $\alg_{\text{BB}}$ and $\alg_{\text{pers}}$.  $\alg_{\text{BB}}$ takes the training data $S_1,\dots,S_\users$ and outputs a representation $h$, and then for each individual $i$, we give them (or they compute for themselves) the model $\hyp_i = \alg_{\text{pers}}(S_i, \rep)$.  So far we have described billboard algorithms as a constraint on the \emph{structure} of the algorithm.  However, conceptually, we think of the representation $h$ as being published, and thus publicly available, while the individual models $\hyp_i$ are computed secretly by each individual and not published.  Thus, when defining privacy for a billboard algorithm, we require that the algorithm $\alg_{\text{BB}}$ be differentially private. Any private billboard algorithm also satisfies joint DP.

Intuitively, there is a hierarchy of frameworks for private personalized learning: the easiest is multitask learning with $1$-out-of-$\users$-DP, then multitask learning with joint DP, then multitask learning with a private billboard algorithm, and finally the hardest is metalearning. 


\mypar{Technical Contributions.} Our technical contributions are a set of relationships between these different frameworks for private personalized learning, which we can now summarize:

\smallskip{\bfseries \em DP metalearning and DP billboard multitask learning are equivalent.} While metalearning can be much harder than multitask learning in general, we show that private multitask learning with billboard algorithms actually implies metalearning. In other words, we recover (a version of) the nonprivate equivalence of \cite{AliakbarpourBBSSU24} when the multitask learner is constrained to produce a publicly visible representation. However, the proof of the implication is quite different.  To gain some intuition for the argument, first observe that the  syntax of algorithms for the two settings is similar: in both cases, we look at all the training datasets to produce a private representation $h$ and then use an individual's data to specialize $h$ to a model in an arbitrary non-private way. The difference lies in which learning objective we expect this representation to satisfy. We prove the implication using the connection between  differential privacy and generalization~\cite{DworkFHPRR14,BassilyNSSSU16,LigettNRWW17,JungLNRSS20}.  Intuitively, in multitask learning, the billboard depends on the training tasks $S_1,\dots,S_\users$ and produces a representation that can be specialized to one of those tasks. However, by privacy, the representation would have almost the same distribution if we had inserted data $S_{\users+1}$ from a fresh training task. Hence, the representation can also be specialized to this unseen test task as well.  The formal result is stated and proved in Section~\ref{sec:metalearning}.

\smallskip{\bfseries \em Separating DP billboard multitask and DP metalearning from JDP multitask learning.} Our first main contribution is to show a separation between JDP multitask learning and billboard multitask learning for an estimation problem that we introduce.  Since our models satisfy a hierarchy from easier to hardest, this result implies a separation between JDP multitask learning and metalearning as well. Such a result has no nonprivate analogue since, absent a privacy constraint, multitask learning implies metalearning \cite{AliakbarpourBBSSU24}. Perhaps not surprisingly, the nonprivate equivalence is inherently privacy-violating---it uses the concatenation of the training data sets as the ``representation''.

We prove the separation 
using what we call the {\bfseries \em indexed mean estimation} problem.  Here, each training datum has the form $(x,j)$ where $x \in \pmo^d$ is a vector and $j \in [d]$ is an index of a coordinate in the vector.  For a task $\dist$, which is a distribution over pairs $(x,j)$, the goal is to output an estimate of the mean of the $j$-th coordinate $\mathbb{E}_{(x,j)\sim \dist}[x_j]$ low mean squared error.  In the multitask learning problem, we will consider $\users$-tuples of tasks with the following constraints: (1) For every individual task $\dist_i$, the marginal distribution of the vector $x$ is identical.  (2) For every individual task $\dist_i$, the marginal distribution of $j$ is deterministic. Under these assumptions, the $x$ part of the distribution has some common mean vector $p$ and and each individual has their own $j_{i}$, and they want an estimate of $p_{j_{i}}$.

Our separations show that DP billboard algorithms for this problem have much higher loss than general JDP algorithms, when the dimension $d$ is larger than the number $\users$ of individuals.  Intuitively, a JDP algorithm can simply give a private estimate of $p_{j_{i}}$ to each individual, which is obtained by averaging $x_{j_{i}}$ over all individuals and adding noise.  Since we only compute $\users$ values in total, the noise variance can be proportional to $\users$.  In contrast, a billboard DP algorithm cannot depend on the specific values $j_{i}$ held by each individual, and thus the billboard must contain enough information to estimate $p_{j}$ for \emph{most} coordinates $j$.  Since the billboard must estimate every coordinate, it must add noise proportional to $d$.  This last statement follows by adapting the  lower bounds of \cite{BunUV14,DworkSSUV15} for marginal estimation.


\smallskip{\bfseries \em Separating $1$-out-of-$\users$ DP from JDP.} We also use the indexed mean estimation problem to separate $1$-out-of-$\users$ and JDP.  In JDP, an attacker sees
estimates of $\users-1$ distinct coordinates of $p$, while in $1$-out-of-$\users$ DP, each individual gets just a single coordinate of $p$ so privacy intuitively requires much less noise.

Our results for indexed mean estimation are summarized in Table~\ref{tab:bound_est}, and stated formally in \Cref{sec:estimation}.

{\bfseries \em Separating DP billboard and JDP for classification.} The separation above applies for an estimation (unsupervised learning) problem.  We also extend our results to prove a separation for a binary classification (supervised learning) problem that we call {\em \bfseries indexed classification}.  In this problem, each task is a distribution over labeled examples $((x,j),y)$ where $x \in \pmo^d$, $j \in [d]$, and $y \in \pmo$.  Our goal is to produce a classifier that predicts $y$ while minimizing the excess classification error.  We  consider distributions in which $y$ is strongly correlated with $x_{j}$ but uncorrelated with other coordinates of $x$. We set these distributions up in such a way that finding a good classifier essentially requires estimating the mean of $x_{j}$.  From here, we proceed similarly to indexed mean estimation.  We give a careful argument that a DP billboard must estimate the mean vector of $x$, denoted $p$, but a JDP algorithm can get away with estimating the mean of a only small number of coordinates of $p$, and thereby introduce less noise.  

This argument is more complex than for indexed mean estimation, since a good billboard algorithm only directly implies good  estimates of the \emph{sign} of each coordinate of $p$. We overcome this by showing that even this simpler problem is hard under differential privacy, via a novel extension of the fingerprinting technique~\cite{DworkSSUV15, BunUV14j}.


Our results for indexed classification are summarized in Table~\ref{tab:bound_est}.  See Section~\ref{sec:classification} for formal statements.

\begin{table*}
    \centering
    \renewcommand{\arraystretch}{1.4}
    \begin{tabular}{|c|c|c||c|c|}
        \hline
        \multirow{2}{4cm}{\centering Privacy Requirement and Learning Objective} & \multicolumn{2}{c||}{Indexed Mean Estimation} &  \multicolumn{2}{c|}{Indexed Classification} \\
         & Upper Bound & 
        Lower Bound & Upper Bound & Lower Bound\\
        \hline
        Nonprivate baseline & \multicolumn{2}{c||}{$\frac{1}{\users}$} &  \multicolumn{2}{c|}{$\frac{1}{\sqrt{\users}}$} \\
        \hline
        1-out-of-$\users$ DP Multitask& $\frac{1}{\eps^2 \users^2} + \frac{1}{\users}$ & - &  $\frac{1}{\eps \users} + \frac{1}{\sqrt{\users}}$ & -\\
        \hline
         JDP Multitask & $\frac{1}{\eps^2\users} + \frac{1}{\users}$
         & $\frac{1}{\eps^2 \users}$  & $\frac{1}{\eps \sqrt{\users}} + \frac{1}{\sqrt{\users}}$ & $\frac{1}{\eps \sqrt{\users}}$\\
         \hline
         DP Billboard Multitask & 
         $\frac{\datadim }{\eps^2 \users^2} + \frac{1}{\users}$
         & $\frac{\datadim}{\eps^2 t^2}$ & $\frac{\sqrt{\datadim}}{\eps\users} + \frac{1}{\sqrt{\users}}$ & $\frac{\sqrt{\datadim}}{\eps\users}$\\
         \hline
         DP Metalearning & $\frac{\datadim}{\eps^2 \users^2} + \frac{1}{\users}$ & $\frac{\datadim}{\eps^2\users^2}$  & $\frac{\sqrt{\datadim}}{\eps \users} + \frac{1}{\sqrt{\users}}$ & $\frac{\sqrt{\datadim}}{\eps\users}$\\
         \hline
    \end{tabular}
    \caption{Asymptotic bounds on the squared-error for Indexed Mean Estimation and excess error for Indexed Classification (assuming $\users \ll \datadim$ and $\samples$ constant).  Here $(\eps,\delta)$ are the privacy parameters, and we assume $\eps \leq 1$ and $\delta = 1/\mathrm{poly}(t)$ for simplicity.} 
    \label{tab:bound_est}
\end{table*}

\subsection{Related Work}
\label{sec:related}

Without privacy constraints, there is a large body of literature on both multitask learning and metalearning, including related concepts or alternative names such as \emph{transfer learning}, \emph{learning to learn}, and \emph{few-shot learning}, which is too vast to survey here.  Another related learning framework is \emph{collaborative learning}~\cite{BlumHPQ17}, which considers multitask classification in a setting where each task has a different marginal distribution on features but there exists a single good labeling function for every task.

There is also related work looking at privacy for model personalization.  The most directly related work is that of  \cite{JainRSST21}, together with predecessors on private recommender systems (e.g. \cite{McSherryM09}), which fits into our framework of JDP multitask learning.  \cite{KricheneJSSTZ2023multi} consider a private multitask learning setting where each individual can contribute data to one task, as opposed to our framework in which tasks perfectly correspond to individuals. 

Private model personalization has some similarities to \emph{federated learning} (see the survey in~\cite{Kairouz+19}), in which the data for each task is stored on a different device, with devices coordinated by a central server, and we want the server to obtain a single good model while minimizing what they can learn about the training data.  Federated learning refers to a distributed system architecture and not a particular learning objective or privacy model. The objective may still be to solve a single learning problem, and there are not necessarily any privacy constraints at all. 

The concept of joint differential privacy was introduced in~\cite{KearnsPRU14}.  Our notion of $1$-out-of-$\users$ DP is also called \emph{marginal differential privacy}~\cite{KannanMRR18}.  The majority of JDP algorithms use the billboard model, which was first formally defined in~\cite{HsuHRRW14}, but was used implicitly in~\cite{McSherryM09,GuptaLMRT10}.

Our lower bound arguments are based on the fingerprinting methodology that was introduced in~\cite{Ullman13,BunUV14} and further refined in~\cite{DworkSSUV15}.  In particular we use (and prove extensions of) the so-called fingerprinting lemmas from~\cite{BunSU17,PeterTU24}.  See~\cite{Kamath2020primer} for a partial survey.

    
    


\section{Privacy Definitions}
\label{sec:privacy-defs}

We now present the formal definitions of privacy that we utilize throughout the paper. 
%
%
We present the definition of standard (central) differential privacy, which is the natural privacy notion for billboard algorithms and for metalearning, as both of these produce a single public output. We then consider privacy models beyond standard (central) DP.

\begin{defn}[Differential Privacy \cite{DworkMNS06, BunS16}]
    \label{def:differential-privacy}
    A randomized algorithm $\alg:\sampledom^\users \to \cO$ satisfies the following differential privacy notions at the person level if, for any pair of $D, D'\in \sampledom^\users$ that differ in one person's dataset,
    \begin{itemize}
        \item $(\eps, \delta)$-DP$:\forall E\subset \cO \quad \pr{}{\alg(D)\in E} \le e^\eps \pr{}{\alg(D')\in E}+\delta$.
        \item $\rho$-zCDP: $\forall \alpha\in(1,\infty), \renyi{\alg(D)}{\alg(D')} \le \rho \alpha$, 
    \end{itemize}
    where $\renyi{\cdot}{\cdot}$ denotes Rényi divergence of order $\alpha$.
\end{defn}

\begin{defn}[Rényi Divergence] For two probability distributions $P$ and $Q$ defined over the same domain, the Rényi divergence of order $\alpha>1$ is 
\begin{equation*}
    \renyi{P}{Q} = \frac{1}{\alpha - 1} \log \mathbb{E}_{x\sim Q} \left(\frac{P(x)}{Q(x)} \right)^\alpha.
\end{equation*}
\end{defn}

Intuitively, differential privacy protects against an adversary from inferring any person $i$'s input even if the adversary knows all the other peoples' inputs as well as the algorithm output. We will say that a billboard algorithm is differentially private if the output of the curator is differentially private.

One can convert between zCDP and DP using the following fact. 

\begin{fact}[Adapted from \cite{BunS16}]\label{fct:zcdp_to_approxDP}
    Let $\alg: \sampledom^{\users} \to \outdom$ satisfy $\rho$-zCDP, then $\alg$ satisfies $(\eps, \delta)$-DP for all $\delta >0$ and 
    \[
    \eps = \rho + 2\sqrt{\rho \log\left(\frac{1}{\delta}\right)}.
    \]
In other words,  for any two  parameters $\epsilon$ and $\delta > 0$, there exists an absolute constant $c$, such that if $\alg$ satisfies $c\cdot\left(\min\left(\epsilon, \frac{\epsilon^2}{\log(1/\delta)}\right)\right)$-zCDP,  then $\alg$ satisfies $(\eps, \delta)$-DP.
\end{fact}

Differentially private mechanisms enjoy a property called \emph{post-processing}. 

\begin{lem}[Post-processing]
For any $(\epsilon, \delta)$-DP algorithm $\alg$ and any (possibly randomized) function $g:\cO \to [0,1]$, for any $D$ and $D'$ that differ in the dataset of one person
\[
\E[g, \alg]{g(\alg(D))} \leq e^\epsilon \E[g, \alg]{g(\alg(D'))} + \delta.
\]
\label{lem:epsdelexp_mult}
\end{lem}

One common way of achieving differential privacy is by using the Gaussian mechanism, which adds noise to the mechanism output calibrated to the \emph{sensitivity} of the mechanism.

\begin{defn}[Sensitivity] A function $q: \cX^n \to \R^d$ has sensitivity $\Delta$ if for all $x, x' \in \cX^n$ differing in a single entry, we have $\| q(x) - q(x')\|_2 \leq \Delta$.
    
\end{defn}
\begin{fact}[Gaussian Mechanism for zCDP \cite{BunS16}]\label{fct:guassianMec}
Let $q: \cX^n \to \R^d$ be a sensitivity-$\Delta$ query. Consider the mechanism $M: \cX^n \to \R^d$ that on input $x$, releases a sample from $\mathcal{N}(q(x), \mathbb{I}_d \sigma^2)$. Then, $M$ satisfies $(\Delta^2/2\sigma^2)$-zCDP.
\end{fact}

Next we consider different models of privacy for algorithms that may produce \emph{different outputs} for different people. The definitions protect against an adversary who is trying to infer the input of some person $i$ and who gets all other $\users -1$ inputs as well as some outputs of the mechanism. We consider other privacy requirements that vary according to what we assume to be visible to an attacker (corresponding to weaker or stronger privacy guarantees).



\mypar{Joint DP}
First, we define Joint DP, where the adversary gets all outputs \emph{except the secret person's output}. 

\begin{defn}[Joint Differential Privacy \cite{KearnsPRU14}]
     A randomized algorithm $\alg \from \sampledom^{\users} \to \outdom^{\users}$
     satisfies $(\eps, \delta)$-JDP if for all people $i \in [\users]$, all pairs of neighboring datasets for person $i$ $S_i, S_i' \in \sampledom$, all datasets for everyone else $S_{-i} \in \sampledom^{\users-1}$, and all $E\subset \cO^{\users-1}$, 
     \(
        \pr{}{\alg(S_i;S_{-i})_{-i}\in E} \le e^\eps \pr{}{\alg(S_i';S_{-i})_{-i}\in E}+\delta
     \)
     If the analogous definition holds with
     \(
        \renyi{\alg(S_i;S_{-i})_{-i}}{\alg(S_i';S_{-i})_{-i}} \le \rho \alpha.
     \)
     instead, then $\alg$ is $\rho$-JCDP.
\end{defn}

\mypar{1-out-of-\users DP}
We also consider a weaker privacy requirement in which the adversary only sees one person's output (not the secret person's). We call this 1-out-of-\users differential privacy (instead of marginal DP \cite{KannanMRR18}).




\begin{defn}[1-out-of-t Differential Privacy]
\label{def:1ot_privacy}
    A randomized algorithm $\alg \from \sampledom^{\users} \to \outdom^{\users}$
     satisfies $(\eps, \delta)$-1-out-of-$\users$-DP (resp. $\rho$-1-out-of-$\users$-zCDP) if for all pairs of distinct people $i,j \in [\users]$, all pairs of neighboring datasets for person $i$ $S_i, S_i' \in \sampledom$, all datasets for everyone else $S_{-i} \in \sampledom^{\users-1}$, and all $E\subset \cO$, \[\pr{}{\alg(S_i;S_{-i})_j\in E} \le e^\eps \pr{}{\alg(S_i';S_{-i})_j\in E}+\delta\]
     Resp.\ $\renyi{\alg(S_i;S_{-i})_j}{\alg(S_i';S_{-i})_j} \le \rho \alpha$  $\forall \alpha \in (1,\infty)$.
\end{defn}

\section{Billboard Multitask Learning Implies 
Metalearning}
\label{sec:metalearning}

In this section, we present a general reduction from multitask learning to metalearning. 
Suppose $\distfam^{(\users)}$ denotes the set of all $\users$-tuples that an algorithm $\MM$ multitask learns. Without loss of generality, we assume that $\distfam^{(\users)}$ is closed under permutation. For any tuple in $\distfam^{(\users)}$, an unseen task could be a similar one that can be swapped with any of the $\users$ tasks, and the new tuple would still remain in the domain. We can apply metalearning to this set of tasks. We first define a notion called \emph{exchangeable distributions}.

\begin{defn}[Exchangeable distributions] A $\users$-tuple of distributions $(\dist_1, \ldots, \dist_\users)\sim \metadist$ is exchangeable with respect to $\metadist$ if for every permutation $\pi: [\users] \to [\users]$, the tuple $(\dist_{\pi(1)}, \ldots, \dist_{\pi(\users)})$ is equally likely to occur in $\metadist$ as the unpermuted one. 
When the metadistribution is clear from the context, we simply refer to the tuple as exchangeable.
\end{defn}

We consider a distribution $\metadist$ over \emph{exchangeable tuples} of length $\users + 1$. Let $\metadist^{(t)}$ be the marginal distribution over $\users$-tuples of distributions obtained by taking the first $\users$ distributions from a $(\users + 1)$-tuple drawn from $\metadist$. If $\distfam^{(t)}$ contains all the $\users$-tuples of distributions that are in the support of $\metadist^{(t)}$, then $\MM$ can metalearn as well. We formalize this argument in the following theorem.

\newcommand{\thmMetaToMulti}{
Fix parameters $\alpha > 0$, $\epsilon > 0$, $\delta \in (0,1)$, $\users \in \mathbb{N}$, $\samples \in \mathbb{N}$, a loss function $\err$ taking values in $[0,1]$, and a distribution $\metadist$ over $(\users + 1)$-tuples of exchangeable distributions. Assume $\distfam^{(t)}$ is a set of $\users$-tuples of distributions that contains all the $\users$-tuples of distributions in the support of $\metadist^{(t)}$ (as defined above). Let $\alg = (\alg_{\text{BB}}, \alg_{\text{pers}})$ be an $(\epsilon, \delta)$-DP billboard algorithm. If $\alg$ multitask learns  $\distfam^{(t)}$ with $\users$ tasks and $\samples$ samples per task 
with expected error $\alpha$, 
then $\alg$ also metalearns $\metadist$ with $\users$ training tasks, $\samples$ samples per training task, and $\samples$ personalization samples, 
with expected error at most 
$e^\epsilon \alpha + \delta$.
}

\begin{thm}
    \thmMetaToMulti
    \label{thm:metatomulti}
\end{thm}

 \begin{proof}We focus on the error of $\alg$ in the metalearning framework. As we have described the error in Definition~\ref{def:metalearning}, the error of this algorithm corresponds to its performance on the $\users+1$-st distribution. More precisely, we have: 
 \[
 \ex{\substack{\alg_{\text{pers}}, \alg_{\text{BB}}, \\ (\dist_1, \ldots, \dist_{\users+1})\sim \metadist,\\ S_1 \sim \dist_1^\samples, \ldots,S_\users \sim \dist_{\users}^\samples, S_{\users+1}\sim \dist_{\users+1}^{\samples}
        }}{ \err(\dist_{\users+1}, \alg_{\text{pers}}(S_{\users+1}, \alg_{\text{BB}}(S_1, \ldots, S_{\users}))}.
 \]

 We know that the distributions in the tuple that we draw from $\metadist$ are exchangeable, so the distribution of $(\users+1)$-tuples of distributions remains identical when we swap $\dist_{\users+1}$ with any $\dist_i$. 
 Therefore, in the metalearning error of $\alg$ we can swap $\dist_{\users+1}, S_{\users+1}$ with any $\dist_i, S_i$ for $i \in [\users]$. Let $D$ be the datasets of the first $\users$ people, $D = (S_1, \ldots, S_{\users})$, and $D_{-i}$ be the $\users$ datasets from $D$ after we swap the dataset of person $i$, $S_i$, with the dataset of person $\users+1$, $S_{\users+1}$. Thus, we have that 
 \begin{align*}
 &\ex{\substack{\alg_{\text{pers}}, \alg_{\text{BB}}, \\ (\dist_1, \ldots, \dist_{\users+1})\sim \metadist,\\ S_1 \sim \dist_1^\samples, \ldots,S_\users \sim \dist_{\users}^\samples, S_{\users+1}\sim \dist_{\users+1}^{\samples}
        }}{ \err(\dist_{\users+1}, \alg_{\text{pers}}(S_{\users+1}, \alg_{\text{BB}}(D))} =
        \frac{1}{\users} \sum_{i \in [\users]} \ex{\substack{\alg_{\text{pers}}, \alg_{\text{BB}}, \\ (\dist_1, \ldots, \dist_{\users+1})\sim \metadist,\\ S_1 \sim \dist_1^\samples, \ldots,S_\users \sim \dist_{\users}^\samples, S_{\users+1}\sim \dist_{\users+1}^{\samples}
        }}{ \err(\dist_{i}, \alg_{\text{pers}}(S_{i}, \alg_{\text{BB}}(D_{-i}))}.
 \end{align*}
The error term in the equation above is analogous to the error term in the multitask learning algorithm, with one key distinction. In multitask learning, the $i$-th error term is computed when $S_i$ is included in the dataset provided to $\alg_{\text{BB}}$. We demonstrate that these two terms are closely related due to the privacy properties of the server's algorithm.

For every individual $i \in [\users]$, $D$ and $D_{-i}$ are neighboring datasets that differ only by the replacement of one person’s data. Consequently, the privacy of $\alg_{\text{BB}}$ ensures that $\alg_{\text{BB}}(D)$ and $\alg_{\text{BB}}(D_{-i})$ are nearly indistinguishable in distribution. It is well established that post-processing does not alleviate this indistinguishability. That is, for any function applied to $\alg_{\text{BB}}(D)$ and $\alg_{\text{BB}}(D_{-i})$, the resulting outputs remain indistinguishable. Specifically, for each individual $i \in [\users]$, given a fixed distribution $\dist_i$ and dataset $S_i$, let $g_i(b)$ be a post-processing function.
 \begin{align}
     g_i(b) = \Exp_{
     \substack{
    \alg_{\text{pers}}}}
     \left[\err(\dist_i, \alg_{\text{pers}}(S_i,b)) \right].
 \end{align}
 We apply \Cref{lem:epsdelexp_mult} for function $g_i$ and obtain that

 \begin{align*}
     \ex{\substack{\alg_{\text{pers}}, \alg_{\text{BB}}, \\ (\dist_1, \ldots, \dist_{\users+1})\sim \metadist,\\ S_1 \sim \dist_1^\samples, \ldots,S_\users \sim \dist_{\users}^\samples, S_{\users+1}\sim \dist_{\users+1}^{\samples}
        }}{ \err(\dist_{i}, \alg_{\text{pers}}(S_{i}, \alg_{\text{BB}}(D_{-i}))} 
        \leq e^{\eps} \cdot \ex{\substack{\alg_{\text{pers}}, \alg_{\text{BB}}, \\ (\dist_1, \ldots, \dist_{\users+1})\sim \metadist,\\ S_1 \sim \dist_1^\samples, \ldots,S_\users \sim \dist_{\users}^\samples, S_{\users+1}\sim \dist_{\users+1}^{\samples}
        }}{ \err(\dist_{i}, \alg_{\text{pers}}(S_{i}, \alg_{\text{BB}}(D))}+\delta
 \end{align*}

 Therefore, the error of $\alg$ in the metalearning setting is bounded by
 \begin{align*}
     \frac{1}{\users} \sum_{i \in [\users]} \ex{\substack{\alg_{\text{pers}}, \alg_{\text{BB}}, \\ (\dist_1, \ldots, \dist_{\users+1})\sim \metadist,\\ S_1 \sim \dist_1^\samples, \ldots,S_\users \sim \dist_{\users}^\samples, S_{\users+1}\sim \dist_{\users+1}^{\samples}
        }}{ \err(\dist_{i}, \alg_{\text{pers}}(S_{i}, \alg_{\text{BB}}(D_{-i}))} \leq
        e^{\eps}  \frac{1}{\users}\sum_{i \in [\users]} \ex{\substack{\alg_{\text{pers}}, \alg_{\text{BB}},\\(\dist_1, \ldots, \dist_{\users+1})\sim \metadist,\\S_1 \sim \dist_1^\samples, \ldots, S_\users \sim \dist_\users^\samples }}{\err(\dist_i, \alg_{\text{pers}}(S_i, \alg_{\text{BB}}(S_1, \ldots, S_\users)))} + \delta.
 \end{align*}
In the second line, the billboard message is computed using data drawn from tasks $1$ to $\users$. This equation captures the multitask learning error for the first $\users$ tasks drawn from $\taskdistr$. Alternatively, this tuple is drawn from $\taskdistr^{(\users)}$, making it a member of $\distfam^{(t)}$. Given the assumptions of the theorem—that $\alg$ is a billboard algorithm that multitask learns $\distfam^{(t)}$ with an error bound of $\alpha$—it follows that for all $(\dist_1, \ldots, \dist_\users) \in \distfam^{(t)}$, we have the following:
\begin{align}
       \frac{1}{\users}\sum_{i \in [\users]} \ex{\substack{\alg_{\text{pers}}, \alg_{\text{BB}}\\S_1 \sim \dist_1^\samples, \ldots, S_\users \sim \dist_\users^\samples }}{\err(\dist_i, \alg_{\text{pers}}(S_i, \alg_{\text{BB}}(S_1, \ldots, S_\users)))} \leq \alpha.
   \end{align}

 Combining the inequalities above we obtain that 
 \begin{align*}
     \ex{\substack{\alg_{\text{pers}}, \alg_{\text{BB}}, \\ (\dist_1, \ldots, \dist_{\users+1})\sim \metadist,\\ S_1 \sim \dist_1^\samples, \ldots,S_\users \sim \dist_{\users}^\samples, S_{\users+1}\sim \dist_{\users+1}^{\samples}
        }}{ \err(\dist_{\users+1}, \alg_{\text{pers}}(S_{\users+1}, \alg_{\text{BB}}(S_1, \ldots, S_{\users}))} \leq e^{\eps} \alpha + \delta.
 \end{align*}
 Thus, the proof is complete. 
 \end{proof}
\section{Indexed Mean Estimation}
\label{sec:estimation}
In this section we showcase how different frameworks for private personalization impact the error achieved by a learning algorithm in an estimation problem. 

\paragraph{Indexed Mean Estimation} 
For simplicity, we first describe the multitask learning version of the estimation problem we focus on. Every person $i$ in $[\users]$ gets $\samples$ samples drawn from distribution $\dist_X$ over domain $\{\pm 1\}^\datadim$ and an index $j_i \in [\datadim]$. $\dist_X$ is a product of $d$ independent distributions over $\{\pm 1\}$, with mean $\mathbf{p} = (p_1, \ldots, p_\datadim )$ in $[-1,+1]^\datadim$. Person $i$ aims to find an estimate, $\hat p_{j_i}$, for the true mean of the $j_i$-th coordinate in distribution $\dist_X$, denoted by $p_{j_i}$. Their loss function is defined as the squared difference between the mean estimate for coordinate $j_i$ and its true mean:
\[
\ell_{\text{mean}}(\dist_X, \hat{p}_{j_i} ) \coloneqq\tfrac 1 4 (\hat{p}_{j_i} - p_{j_i})^2,
\]
(where the $\tfrac 1 4$ factor keeps the loss in $[0,1]$).

To cast this problem in the language of multitask learning, as in \Cref{def:multitask_learning}, we can think of 
the equivalent setting where the input of each person comes from a distribution $\dist_i$. Instead of assuming that each person has $n$ i.i.d.\ samples from a shared $\dist_X$ and an index $j_i$, we define $\dist_i$ to be the Cartesian product of $\dist_X$ and a singleton distribution on $\{j_i\}$. In this new setting, the input of person $i$ is $\samples$ samples drawn from $\dist_i$. 

 We define the domain of $\users$-tuples of distributions for $\users$ people that use the same feature distribution $\dist_X$, but potentially different indices as the following 
\begin{equation}
    \label{eq:distfam_est}
\distfam_{\text{est},\datadim,\users} \coloneqq \{(\dist_1, \ldots, \dist_\users)\mid \dist_i \coloneqq \dist_X(\mathbf{p})\otimes \Det(j_i)
\forall i \in [\users],
 \forall \mathbf{p} \in [-1,+1]^\datadim, \forall j_1, \ldots,j_\users  \in [\datadim]^\users \}\,,
\end{equation}
where $\dist_X(\mathbf{p})$ is the distribution $\dist_X$ with mean $\mathbf{p}$ and $\Det(j_i)$  deterministically returns $j_i$.




\subsection{Upper bounds}

We establish upper bounds for indexed mean estimation in the multitask learning frameworks.
All three bounds result from applying the Gaussian mechanism to the empirical means with variance adjusted  according to the privacy requirements. 


\newcommand{\lemMeanEstNpUB}{ For any $\datadim \in \mathbb{N}$ and $\users \in \mathbb{N}$, and loss function $\ell_{\text{mean}}$, there exists an algorithm that multitask learns $\distfam_{\text{est}, \datadim, \users}$ with error $1/4t$ with $\users$ tasks and one sample per task .}

\begin{prop}[Nonprivate upper bound]
   \lemMeanEstNpUB
   \label{prop:meanestnpub}
\end{prop} 
\begin{proof}

The algorithm takes as input one sample $(x^{(i)},j^{(i)})$ drawn from $\dist_i$ for every person $i \in [\users]$ and then computes the empirical mean $\bar{p} \gets \frac{1}{t} \sum_{i\in[t]} x^{(i)}$. Note that for all $\ell \in [t]$ and $j \in [\datadim]$, $x^{(\ell)}_j$ is independent of the other coordinates and $(x^{(\ell)}_j)^2 = 1$. Therefore, the error is 
\begin{align*}
    &\frac{1}{4\users}\sum_{i \in [\users]}\ex{}{\left({\bar{p}_{j_i}-p_{j_i}}\right)^2} 
    = \frac{1}{4\users}\sum_{i \in [\users]} \var{}{\bar{p}_{j_i}-p_{j_i}}
    =\frac{1}{4\users}\sum_{i \in [\users]} \var{}{\frac{1}{t} \sum_{\ell\in[t]} x^{(\ell)}_{j_i}-p_{j_i}}\\
    & = \frac{1}{4\users}\sum_{i \in [\users]} \frac{1}{\users^2} \sum_{\ell \in [\users]}\var{}{x^{(\ell)}_{j_i}} \leq  \frac{1}{4t}
    \,.
\end{align*}
\end{proof}

For our least stringent privacy setting, $1$-out-of-\users privacy, we incur only an additional $O(1/\users^2)$ additive factor to the error above the nonprivate baseline.

\newcommand{\thmMeanEstoootUB}{ 
For any  $\rho \geq 0$, $d \in \mathbb{N}$, and $t \in \mathbb{N}$, and loss function $\ell_{\text{mean}}$, there exists an algorithm that multitask learns $\distfam_{\text{est}, \datadim, \users}$ with error \[\frac{1}{2\rho t^2} + \frac{1}{4t}\] with $\users$ tasks and one sample per task and satisfies $\rho$-1-out-of-$\users$-zCDP. Furthermore, for any  $\epsilon > 0$, $\delta \in (0,1)$, $d \in \mathbb{N}$, and $t \in \mathbb{N}$, there exists an algorithm that multitask learns $\distfam_{\text{est}, \datadim, \users}$ with error $\Theta( \min(\epsilon, {\epsilon^2}/{\log(1/\delta)})^{-1})\cdot \frac{1}{t^2} + \frac{1}{4t}$ with $\users$ tasks and one sample per task, and satisfies $(\epsilon, \delta)$-1-out-of-$\users$-DP.}

\begin{thm}[$1$-out-of-$t$ upper bound]
   \thmMeanEstoootUB
   \label{thm:meanestoootub}
\end{thm}
\begin{proof} 
    We propose the following algorithm, which begins with the empirical mean and produces a private version of the mean for the coordinate of each person $i \in [\users]$ using the Gaussian mechanism. More precisely, the personalized output for person $i$ with index $j_i$ is
    \begin{align*}
        &\hat{p}_{j_i}\gets \bar{p}_{j_i} + Z_{j_i}, 
    \end{align*} 
    where $\bar{p}_{j_i} \gets \frac{1}{t}\sum_{k \in [\users]} x^{(k)}_{j_i}$ and $Z_{j_i}\sim \cN\left(0, \frac{2}{\rho t^2}\right)$.
    
    To show that this algorithm satisfies $\rho$-$1$-out-of-$\users$-zCDP, fix a pair of distinct people $i \neq j$ and focus on the personalized outputs received by person $k$ while replacing the input dataset of person $i$ with a neighboring one. Since each person has a dataset of only one sample (i.e., $\samples = 1$), two neighboring datasets correspond to having two different samples. Now, if we change the sample of person $i$, the sensitivity of the mean at index $j_k$, corresponding to the output of person $k$, is bounded by
    
    \[\abs{\bar{p}_{j_k}(D)-\bar{p}_{j_k}(D')}\le \frac{2}{t}.\]
    Thus, using Gaussian mechanism as outlined in Fact~\ref{fct:guassianMec}, $\hat{p}_{j_k}$ satisfy $\rho$--zCDP for all distinct $i$ and $k$. Therefore, our algorithm satisfies $\rho$-$1$-out-of-$\users$-zCDP. 
    
    Next, we analyze the error of our estimates. Since the Gaussian noise and the sampled data are independent, we have 
    \begin{align*}
        &\frac{1}{4\users}\sum_{i \in [\users]}\ex{}{\left(\hat{p}_{j_i} - p_{j_i}\right)^2} = \frac{1}{4\users}\sum_{i \in [\users]} \var{}{\hat{p}_{j_i}-p_{j_i}} \\
        &= \frac{1}{4\users}\sum_{i \in [\users]} \left(\var{}{\bar{p}_{j_i}-p_{j_i}} + \var{}{Z_{j_i}}\right) \leq \frac{1}{4t}+\frac{1}{2\rho t^2}\,.
    \end{align*}
    The last inequality follows from the error analysis in the proof of Theorem~\ref{prop:meanestnpub}. Using the standard conversion between zCDP and approximate DP, from Fact~\ref{fct:zcdp_to_approxDP}, we show that this algorithm also satisfies $(\eps, \delta)$-DP by setting $\rho = \Theta \big(\max\big(\frac{1}{\epsilon}, \frac{\log(1/\delta)}{\epsilon^2}\big)\big)$.
\end{proof}

Increasing the privacy requirement from 1-out-of-\users privacy to JDP, the error term due to privacy incurs an additional $O(1/\users)$ factor. 

\newcommand{\thmMeanEstJDPUB}{For any $\rho \geq 0$, $d \in \mathbb{N}$, and $t \in \mathbb{N}$, and loss function $\ell_\text{mean}$, there exists an algorithm that multitask learns $\distfam_{\text{est}, \datadim, \users}$ with error \[O\left( \frac{1}{\rho t}\right) + \frac{1}{4t}\] with $\users$ tasks and one sample per task and satisfies $\rho$-JCDP.
Moreover, for any $\eps > 0$, $\delta \in (0,1)$, $d \in \mathbb{N}$, and $t \in \mathbb{N}$, and loss function $\ell_\text{mean}$, there exists an algorithm that multitask learns $\distfam_{\text{est}, \datadim, \users}$ with error $\Theta( \min(\epsilon, {\epsilon^2}/{\log(1/\delta)})^{-1})\cdot\frac{1}{t}  + \frac{1}{4t}$ with $\users$ tasks and one sample per task, and satisfies $(\epsilon, \delta)$-JDP.}

\begin{thm}[JDP upper bound]
   \thmMeanEstJDPUB
   \label{thm:meanestjdpub}
\end{thm}
\begin{proof} Here, for each person $i$ with index $j_i$, the algorithm returns an estimate 
\begin{align*}
    \hat{p}_{j_i} \gets \bar{p}_{j_i}+Z_{j_i},
\end{align*}
where $\bar{p}_{j_i} \gets \frac{1}{t}\sum_{k \in [\users]} x^{(k)}_{j_i}$ and $Z_{j_i} \sim \cN\left(0, \frac{2(t-1)}{\rho t^2}\right)$.
This satisfies $\rho$-JCDP because the global $\ell_2$ sensitivity of the outputs given to the $\users - 1$ colluding people $(\bar{p}_{j_1}, \ldots, \bar{p}_{j_{t-1}})$ is $\frac{2\sqrt{t-1}}{t}$. 
Hence, on average over the $\users$ people participating 
\begin{align*}
&\frac{1}{4\users}\sum_{i \in [\users]}\ex{}{\left({\hat{p}_{j_1}-p_{j_1}}\right)^2} = 
\frac{1}{4\users}\sum_{i \in [\users]}\var{}{\hat{p}_{j_i}-p_{j_i}} \\
&=  \frac{1}{4\users}\sum_{i \in [\users]}\left(\var{}{\bar{p}_{j_1} - p_{j_1}}+\var{}{Z_{j_1}}\right) \leq  \frac{1}{4t}+\frac{(t-1)}{2\rho t^2} 
\,.
\end{align*}

Using the standard conversion between zCDP and approximate DP, from Fact~\ref{fct:zcdp_to_approxDP}, we show that this algorithm also satisfies $(\eps, \delta)$-DP by setting $\rho = \Theta \big(\max\big(\frac{1}{\epsilon}, \frac{\log(1/\delta)}{\epsilon^2}\big)\big)$.
\end{proof}
Lastly, in the billboard model, the privacy error term scales up with the dimension $\datadim$ of the data. 

\newcommand{\thmMeanEstBBUB}{For any $\rho \geq 0, \datadim \in \mathbb{N}, \users \in \mathbb{N}$ and for loss function $\ell_\text{mean}$, there exists a billboard algorithm that multitask learns $\distfam_{\text{est},\datadim, \users}$ with error \[\frac{d}{2\rho t^2} + \frac{1}{4t}\] with $\users$ tasks and $1$ sample per task and satisfies $\rho$-zCDP.
Moreover, for any $\eps > 0, \delta \in (0,1), \datadim \in \mathbb{N}, \users \in \mathbb{N}$ and for loss function $\ell_\text{mean}$,
there exists an algorithm in the billboard model that multitask learns $\distfam_{\text{est}, \datadim, \users}$ to error $\Theta( \min(\epsilon, {\epsilon^2}/{\log(1/\delta)})^{-1})\cdot\frac{d}{t^2}  + \frac{1}{4t}$ with $\users$ tasks and $1$ sample per task,  and satisfies $(\epsilon, \delta)$-DP.}

\begin{thm}[Billboard upper bound]
    \thmMeanEstBBUB
    \label{thm:meanestbbub}
\end{thm}
\begin{proof}
     The billboard algorithm runs the Gaussian mechanism once and broadcasts the output to all the people. Then every person can use their index $j_i$ to get the estimate they are interested in.
    More specifically, the algorithm takes as input one sample $(x^{(i)},j^{(i)})$ drawn from $\dist_i$ from every person $i$ and then computes an empirical mean $\bar{p}\coloneqq \frac{1}{t} \sum_{i\in[t]} x^{(i)}$ and outputs on the billboard
    \begin{align}
        \hat{p} \gets \bar{p} + Z, \quad \text{where } Z\sim \cN\left(0, \frac{2d}{\rho t^2} \bbI_{\datadim}\right).
    \end{align}
    The global $\ell_2$ sensitivity of $\bar p$ is $\frac{2\sqrt{d}}{t}$, so this satisfies $\rho$-zCDP and, consequently, $\rho$-JCDP.

    To analyze the error, by symmetry it suffices to analyze person $i$.
    Since sampling the data and the Gaussian noise are independent, we have
    \begin{align}
        \ex{}{\left(\hat{p}_{j_i} - p_{j_i}\right)^2} = 
        \var{}{\hat{p}_{j_i} - p_{j_i}} &= \var{}{Z_{j_i}} + \var{}{\bar{p}_{j_i} - p_{j_i}}.
    \end{align}
    We have $\var{}{Z_{j_i}}=\frac{d}{2 \rho t^2}$ and $\var{}{\bar{p}_{j_i} - p_{j_i}}\le \frac{1}{t}$.
    Summing these variances gives you that the total error is 
    \[
    \frac{1}{4\users}\sum_{i \in [\users]} \ex{}{(\hat p_{j_i} - p_{j_i})^2} \leq \frac{d}{ 2\rho t^2}+\frac{1}{4t}.
    \]

    Using the standard conversion between zCDP and approximate DP, from Fact~\ref{fct:zcdp_to_approxDP}, we show that this algorithm also satisfies $(\eps, \delta)$-DP by setting $\rho = \Theta \big(\max\big(\frac{1}{\epsilon}, \frac{\log(1/\delta)}{\epsilon^2}\big)\big)$.
\end{proof}
This problem can easily be extended to metalearning for a metadistribution over $\distfam_{\text{est},\datadim,\users+1}$ with $\users$ training tasks and a test task with loss function $\ell_\text{mean}$. \Cref{thm:meanestbbub}, combined with our general reduction from metalearning to multitask learning in \Cref{thm:metatomulti}, also implies the existence of a metalearning algorithm within the billboard model. 

\begin{cor}[Metalearning upper bound]
For any $\eps >0, \delta \in (0,1), \datadim \in \mathbb{N}, \users \in \mathbb{N}$ and for loss function $\ell_\text{mean}$, there exists a pair of algorithms $(\alg_{\text{meta}}, \alg_{\text{pers}})$ that metalearn a distribution $\metadist$ over $\distfam_{\text{est},\datadim, \users+1}$ with error $\Theta( e^\epsilon\cdot \min(\epsilon, {\epsilon^2}/{\log(1/\delta)})^{-1})\cdot\frac{d}{t^2}  + \frac{e^\epsilon}{4t} + \delta$ using $\users$ training tasks, one sample per training task and a test task with one personalization sample and $\alg_{\text{meta}}$ satisfies $(\epsilon, \delta)$-DP.
\end{cor}

\subsection{Lower bounds}

Next, we formally separate the privacy frameworks in the multitask setting by proving a number of lower bounds. The proofs for the JDP multitask lower bound and the metalearning lower bound use the fingerprinting lower bound for private mean estimation. Specifically, we adapt the tracing attack used in the proof of the standard lower bound to each of our threat models. The proofs are provided in \Cref{sec:estimation_proofs}.

\newcommand{\thmMeanEstJDPLB}{Fix $\alpha >0$, $\users \in \mathbb{N}$, $\datadim \geq c \users^2$ for a sufficiently large constant, $\eps > 0$ and $\delta\in (0,\frac{1}{96t})$. Let $\ell_\text{mean}$ be the loss function. Let $\alg$ be an algorithm that multitask learns $\distfam_{\datadim, \users}$ with error $\alpha$ with $\users$ tasks and one sample per task and satisfies $(\eps, \delta)$-JDP.
   Then, \[\alpha \ge \Omega\left(\min\left\{ \frac{1}{{\eps^2\users}},1\right\}\right).\]}
\begin{thm}[JDP lower bound]
    \thmMeanEstJDPLB
    \label{thm:mean-JDP-lb}
\end{thm}

\newcommand{\thmMeanEstMetaLB}{
Fix $\alpha > 0$, $\eps > 0$, $\users \in \mathbb{N}$, $\datadim \in \mathbb{N}$ and $\delta \in (0, \frac{1}{96t})$. Let $\ell_\text{mean}$ be the loss function.
Let $\alg = (\alg_{\text{meta}}, \alg_{\text{pers}})$ be a pair of algorithms that metalearn a distribution $\metadist$ over $\distfam_{\text{est},\datadim, \users+1}$ with error $\alpha$ with $\users$ training tasks, one sample per training task and a test task with one personalization sample and $\alg_{\text{meta}}$ satisfies $(\eps, \delta)$-JDP.
Then, \[\alpha \ge \Omega\left(\min\left\{\frac{d}{\epsilon^2 t^2},1\right\}\right).\]}

\begin{thm}[Metalearning lower bound]
\thmMeanEstMetaLB
\label{thm:meanestmetalb}
\end{thm}

 For the multitask learning lower bound in the billboard model, we use \Cref{thm:metatomulti} to reduce to the metalearning lower bound in \Cref{thm:meanestmetalb}.

\newcommand{\thmMeanEstBBLB}{
Fix $\alpha > 0$, $\eps \in (0,1]$, $\users \in \mathbb{N}$, $\delta\in (0, \frac{1}{8 \cdot 144t})$ and $d\ge c\users$ for a sufficiently large constant $c$. Let $\ell_\text{mean}$ be the loss function. Let $\alg$ be a billboard algorithm that multitask learns $\distfam_{\datadim, \users}$ with error $\alpha$ with $\users$ tasks and one sample per task, and satisfies $(\eps, \delta)$-DP.
Then, \[\alpha \ge \Omega\left(\min\left\{\frac{d}{\epsilon^2 t^2},1\right\}\right).\]}
\begin{cor}[Billboard lower bound]
\thmMeanEstBBLB
\label{thm:meanestbblb}
\end{cor}

\section{Indexed Classification}
\label{sec:classification}


In this section, we characterize the error achievable on a classification problem in different private personalization frameworks.
\paragraph{Indexed Classification } Every individual $i$ has an index $j_i \in [\datadim]$ and $\samples$ labeled samples $\{(x^{(i,k)}, y^{(i,k)})\}_{k \in [\samples]}$ drawn from a personal distribution $R_{\mathbf{p},j_i}$ over $\{\pm 1\}^\datadim \times \{\pm 1\}$, where $\mathbf{p}$ is a parameter of the data distribution that is the same among individuals. For a fixed vector $\mathbf{p} = (p_1, \ldots, p_{\datadim})$ in $[-1,+1]^\datadim$ and an index $j \in [\datadim]$, we define $R_{\mathbf{p}, j}$ as the distribution on the outcome of the following sampling process. First, draw an auxiliary random variable $w \in \{\pm 1\}^\datadim$ from a product of $\datadim$ independent distributions with mean $\mathbf{p}$. Then, draw $y \in \{\pm 1\}$ uniformly. Finally, construct $x$ as follows
\begin{align*}
    x_\ell = \begin{cases}
        w_\ell , &\textbf{ if } \ell \neq j\\
        w_\ell\cdot y, &\textbf{ if } \ell = j.
    \end{cases}
\end{align*}

Individual $i$'s goal is to use these samples, and in cooperation with other individuals learn a classifier of the form $\hat{f}_i(x_{j_i}): \{\pm 1\} \to \{\pm 1\}$ that predicts labels $y$ as correctly as possible using only the $j_i$-th feature. Depending on the output structure of the algorithm, $\hat{f}_i$ might be sent as a private message from the curator or it might be calculated locally by individual $i$ using the representation broadcasted by the curator. The loss function of individual $i$ is the difference between the misclassification error of $\hat{f}_i$ and the optimal classifier for distribution $R_{j_i}$:
\begin{align*}
\ell_{\text{class}}(R_{\mathbf{p},j_i}, \hat{f}_i) \coloneqq \pr{}{\hat{f}_i(x_{j_i}) \neq y} - \min_{f_i} \pr{}{f_i(x_{j_i}) \neq y},
\end{align*}
where the probability is taken over $(x,y) \sim R_{\mathbf{p},j_i}$.

We cast this problem in the "multitask learning" language of \Cref{def:multitask_learning} and specify the connection between the data distributions by considering that every individual has $\samples$ copies of index $j_i$ and that the $\samples$ samples come from the following family of $\users$-tuples of distributions
\begin{align*}
    \distfam_{\text{class}, \datadim, \users} \coloneqq \{(\dist_1, \ldots, \dist_\users) \mid \dist_i = R_{\mathbf{p}, j_i} \otimes \Det(j_i) \forall i \in [\users],
    \forall j_i \in [\datadim], \forall \mathbf{p} \in [-1,+1]^\datadim\}.
\end{align*}
In this definition, for all $\dist_i$ that are in the same tuple, the distributions over the features, $R_{\mathbf{p}, j_i}$ are parameterized by the same $\mathbf{p}$. However, the $\dist_i$s can potentially have different $j_i$s for different people. For convenience we write $(x,j,y) \sim \dist_i$, instead of $(x,y,j)$.

\subsection{Indexed Sign Estimation}
For the results of this section, we first prove a reduction to and from an estimation problem we call \textit{indexed sign estimation}. We then obtain upper and lower bounds for the error of indexed sign estimation that imply bounds for indexed classification.

\paragraph{Indexed sign estimation problem} In the multitask learning version of indexed sign estimation there are $\users$ people and every person $i \in [\users]$ has $\samples$ samples drawn from the same distribution $P_X$ over domain $\{\pm 1\}^\datadim$ and an index $j_i \in [\datadim]$. As in \Cref{sec:estimation}, $\dist_X$ is a product of $\datadim$ independent distributions over $\{\pm1\}^\datadim$ with mean $\mathbf{p} \in [-1,+1]^\datadim$. Person $i$'s goal is to estimate the sign $\hat{s}_{j_i}$ of the mean of coordinate $j_i$. The error for person $i$ is $|p_{j_i}|$ if they make a wrong prediction of the sign of $p_{j_i}$ and $0$ otherwise. This means that they get penalized more for the sign of coordinates that have mean further from zero and are, thus, easier to predict. The loss function for person $i$ is 
\[
\ell_{\text{sign}}(\dist_X, \hat{s}_{j_i}) \coloneqq \ind\{\text{sign}(p_{j_i}) \neq \hat{s}_{j_i}\}|p_{j_i}|,
\]
where $p_{j_i}$ is the true mean of coordinate $j_i$ and $\hat{s}_{j_i}$ is the sign estimate of person $i$.
The overall error for multitask learning is
\[
\frac{1}{\users}\sum_{i \in [\users]} \ex{}{\ind\{\text{sign}(p_{j_i}) \neq \hat{s}_{j_i}\}|p_{j_i}|},
\]
where the expectation is taken over the samples of the $\users$ people and the randomness of the algorithm.
Following the same process as in \Cref{sec:estimation} to cast this problem in the ``multitask learning'' language, we consider that every person draws $\samples$ samples from a distribution $\dist_i$ that is in a $\users$-tuple of distributions from family $\distfam_{\text{est}, \datadim ,\users}$, as defined in \Cref{sec:estimation}.

We start by showing that we can reduce indexed classification to indexed sign estimation in the multitask learning framework, and vice versa.

\newcommand{\thmRedClasstoEst}{ Fix $a \in [0,1]$. Let $\alg$ be an algorithm that multitask learns $\distfam_{\text{est}, \datadim, \users}$ with error $\alpha$, for loss function $\ell+\text{sign}$, with $\users$ tasks and $1$ sample per task.
Then, there exists an algorithm that multitask learns $\distfam_{\text{class}, \datadim, \users}$ with error $\alpha$, for loss function $\ell_\text{class}$, with $\users$ tasks and $1$ sample per task.
 }
\begin{lem}[Reduction of indexed classification to indexed sign estimation]
    \thmRedClasstoEst   
    \label{thm:redclasstoest}
\end{lem}
\begin{proof}
    Let $\{(x^{(i)}, j^{(i)}, y^{(i)})\}_{i \in [\users]}$ be the input to the multitask classification algorithm that is drawn from a tuple of distributions $(\dist_{\text{class}}^{(1)}, \ldots, \dist_{\text{class}}^{(\users)})$ in $\distfam_{\text{class}, \datadim, \users}$. We can transform this dataset to an input dataset for algorithm $\alg$. For person $i$ we can use $j^{(i)} = j_i$ to recover the value of the auxiliary random variable $w^{(i)}$ as follows
    \[
    w^{(i)}_\ell \gets \begin{cases}
        x^{(i)}_\ell, &\text{ if } \ell \neq j^{(i)},\\
        x^{(i)}_\ell y^{(i)}, &\text{ if } \ell = j^{(i)}.\\
    \end{cases}
    \]
    We see that the tuple of the $\users$ distributions $(\dist^{(1)}_{\text{est}}, \ldots,\dist^{(\users)}_{\text{est}})$ over $(w^{(i)},j^{(i)})$ for all $i \in [\users]$ is in $\distfam_{\text{est}, \datadim, \users}$. For this reason, we give $\{(w^{(i)}, j^{(i)})\}_{i \in [\users]}$ as input to algorithm $\alg$.

    Let $\hat{s}_{j_i}$ be the output of $\alg$ for person $i$. Then, for person $i \in [\users]$ the classification algorithm ouputs the labeling function
    \[
        \hat{f}_i(x_{j_i}) = \hat{s}_{j_i} x_{j_i}.
    \]
    
    The misclassification error of $\hat{f}_i$ for distribution $\dist_{\text{class}}^{(i)}$ is 
    \begin{align*}
        \pr{(x,j,y) \sim \dist_{\text{class}}^{(i)}}{\hat{f}_i(x_j) \neq y} & =\pr{(x,j,y) \sim \dist_{\text{class}}^{(i)}}{\hat{s}_{j_i}x_{j_i} \neq y}\\
        & = \pr{(x,j,y) \sim \dist_{\text{class}}^{(i)}}{\hat{s}_{j_i} \neq y x_{j_i}}\\
        & = \pr{(x,j,y) \sim \dist_{\text{class}}^{(i)}}{\hat{s}_{j_i} \neq w_{j_i}}\\
        & = \pr{(w,j) \sim \dist_{\text{est}}^{(i)}}{\hat{s}_{j_i} \neq w_{j_i}}\\
        & = \ind\{\hat{s}_{j_i} = 1\}\pr{(w,j) \sim \dist_{\text{est}}^{(i)}}{w_{j_1} = -1}+\ind\{\hat{s}_{j_i} = -1\}\pr{(w,j) \sim \dist_{\text{est}}^{(i)}}{w_{j_1} = 1}\\
        & = \frac{1 - p_{j_i}}{2} \ind\{\hat{s}_{j_i} = 1\}+\frac{1+p_{j_i}}{2}\ind\{\hat{s}_{j_i} = -1\}.
    \end{align*}
    The misclassification error of a fixed classifier $f_i$ for distribution $\dist_{\text{class}}^{(i)}$ is
    \begin{align*}
         \pr{(x,j,y) \sim \dist_{\text{class}}^{(i)}}{f_i(x_j) \neq y} 
         ={} &\pr{(x,j,y) \sim \dist_{\text{class}}^{(i)}}{f_i(x_j) = -1 \mid x_{j_i} = 1, y = 1}\pr{}{x_{j_i} = 1, y = 1}\\
         &+ \pr{(x,j,y) \sim \dist_{\text{class}}^{(i)}}{f_i(x_j) = 1 \mid x_{j_i} = 1, y = -1}\pr{}{x_{j_i} = 1, y = -1}\\
         &+\pr{(x,j,y) \sim \dist_{\text{class}}^{(i)}}{f_i(x_j) = -1 \mid x_{j_i} = -1, y = 1}\pr{}{x_{j_i} = -1, y = 1}\\
         &+\pr{(x,j,y) \sim \dist_{\text{class}}^{(i)}}{f_i(x_j) = 1 \mid x_{j_i} = -1, y = -1}\pr{}{x_{j_i} = -1, y = -1}\\
         ={} &\pr{(x,j,y) \sim \dist_{\text{class}}^{(i)}}{f_i(x_j) = -1 \mid x_{j_i} = 1, y = 1}\frac{1+p_{j_i}}{4}\\
         &+ \pr{(x,j,y) \sim \dist_{\text{class}}^{(i)}}{f_i(x_j) = 1 \mid x_{j_i} = 1, y = -1}\frac{1-p_{j_i}}{4}\\
         &+\pr{(x,j,y) \sim \dist_{\text{class}}^{(i)}}{f_i(x_j) = -1 \mid x_{j_i} = -1, y = 1}\frac{1-p_{j_i}}{4}\\
         &+\pr{(x,j,y) \sim \dist_{\text{class}}^{(i)}}{f_i(x_j) = 1 \mid x_{j_i} = -1, y = -1}\frac{1+p_{j_i}}{4}.\\
    \end{align*}
    Given $p_{j_i}$, the function that minimizes this error is $f^*_i(x_{j_i}) = \text{sign}(p_{j_i}) x_{j_i}$. The misclassification error of $f^*_i$ for distribution $\dist_{\text{class}}^{(i)}$ is 
    \begin{align*}
        \pr{(x,j,y) \sim \dist_{\text{class}}^{(i)}}{f^*_i(x_j) \neq y} & =\pr{(x,j,y) \sim \dist_{\text{class}}^{(i)}}{\text{sign}(p_{j_i})x_{j_i} \neq y}\\
        & =\pr{(w,j) \sim \dist_{\text{est}}^{(i)}}{\text{sign}(p_{j_i}) \neq w_{j_i}}\\
        & = \ind\{\text{sign}(p_{j_i}) = 1\} \frac{1-p_{j_i}}{2}+\ind\{\text{sign}(p_{j_i}) = -1\} \frac{1+p_{j_i}}{2}.
    \end{align*}

    Taking the difference between the misclassification error of $\hat{f}_i$ and $f*_i$, we obtain that 
    \begin{align*}
        &\pr{(x,j,y) \sim \dist_{\text{class}}^{(i)}}{\hat{f}_i(x_j) \neq y} - \pr{(x,j,y) \sim \dist_{\text{class}}^{(i)}}{f^*_i(x_j) \neq y} \\
        ={} &\frac{1 - p_{j_i}}{2} \ind\{\hat{s}_{j_i} = 1\}+\frac{1+p_{j_i}}{2}\ind\{\hat{s}_{j_i} = -1\}
        - \ind\{\text{sign}(p_{j_i}) = 1\} \frac{1-p_{j_i}}{2}-\ind\{\text{sign}(p_{j_i}) = -1\} \frac{1+p_{j_i}}{2}\\
        ={} &\ind\{\hat{s}_{j_i} = 1, \text{sign}(p_{j_i}) = -1\}(-p_{j_i})+\ind\{\hat{s}_{j_i} = -1, \text{sign}(p_{j_i}) = 1\}p_{j_i}\\
        ={} &\ind\{\hat{s}_{j_i} \neq \text{sign}(p_{j_i}) \}|p_{j_i}|.
        \end{align*}

    Therefore, if we take the expectation of the error over the samples $\{(x^{(i)}, j^{(i)}, y^{(i)})\}_{i \in [\users]}$ and the randomness of algorithm $\alg$ we have that the error of multitask learning for indexed classification is 
    \begin{align*}
        \frac{1}{\users} \sum_{i \in [\users]} \ex{}{\left( \pr{(x,j,y) \sim \dist_i}{\hat{f}_i(x_{j_i}) \neq y} - \min_{f_i} \pr{(x,j,y) \sim \dist_i}{f_i(x_{j_i}) \neq y}\right)} &\leq \frac{1}{\users} \sum_{i \in [\users]} \ex{}{\ind\{\hat{s}_{j_i} \neq \text{sign}(p_{j_i}) \}|p_{j_i}|} 
        \leq \alpha.
    \end{align*}
    This concludes our proof.
\end{proof}

\newcommand{\thmRedEsttoCLass}{Fix $\alpha \in [0,1]$. Let $\alg$ be an algorithm that multitask learns $\distfam_{\text{class}, \datadim, \users}$ with error $\alpha$, for loss function $\ell_\text{class}$, with $\users$ tasks and $1$ sample per task. Then, there exists an algorithm that multitask learns $\distfam_{\text{est}, \datadim, \users}$ with error $\alpha$, for loss function $\ell_\text{est}$, with $\users$ tasks and $2$ samples per task.
}

\begin{lem}[Reduction of indexed sign estimation to indexed classification]
    \thmRedEsttoCLass
    \label{thm:redesttoclass}
\end{lem}

\begin{proof}
    Let $\{(x^{(i,1)},j^{(i,1)}), (x^{(i,2},j^{(i,2)})\}$ be the dataset of task $i \in [\users]$ gives to the indexed estimation algorithm. Every sample $(x^{(i,k)},j^{(i,k)})$, for $k \in \{1,2\}$, is drawn from the corresponding distribution $\dist_{\text{est}}^{(i)}$ from a tuple of $\users$ distributions $(\dist_{\text{est}}^{(1)}, \ldots, \dist_{\text{est}}^{(\users)}) \in \distfam_{\text{est}, \datadim, \users}$. 

    We first use one sample from every task/person to construct an input for algorithm $\alg$. For every person $i$ we set $w^{(i,1)} \gets x^{(i,1)}$, draw a $y^{(i,1)} \in \{\pm 1\}$ uniformly and set 
    \[
    \tilde{x}^{(i,1)}_\ell \gets \begin{cases}
        w^{(i,1)}_\ell, &\text{ if } \ell \neq j_i\\
        w^{(i,1)}_\ell y^{(i,1)}, &\text{ if } \ell = j_i.\\
    \end{cases}
    \]
    We then give $\{(\tilde{x}^{(i,1)}, j^{(i,1)}, y^{(i,1)})\}_{i \in [\users]}$ as input to the classification algorithm $\alg$. Let $\dist_{\text{class}}^{(i)}$ be the distribution of $(\tilde{x}^{(i,1)}, j^{(i,1)}, y^{(i,1)})$. We notice that $(\dist_{\text{class}}^{(1)}, \ldots, \dist_{\text{class}}^{(\users)})$ is in the family of tuples $\distfam_{\text{class}, \datadim, \users}$. 

    Let $\hat{f}_1, \ldots \hat{f}_\users$ be the functions that $\alg$ outputs for this input. Then, we use the second sample of person $i$, $(x^{(i,2},j^{(i,2)})$, and function $\hat{f}_i$ to get the sign estimate $\hat{s}_{j_i}$. In particular, we draw $y^{(i,2)} \in \{\pm 1\}$ uniformly, we set $w^{(i,2)} \gets x^{(i,2)}$ and 
    \[
    \tilde{x}^{(i,2)}_\ell \gets \begin{cases}
        w^{(i,2)}_\ell, &\text{ if } \ell \neq j_i\\
        w^{(i,2)}_\ell y^{(i,1)}, &\text{ if } \ell = j_i.\\
    \end{cases}
    \]
    We see that $\dist_{\text{class}}^{(i)}$ is also the distribution of $(\tilde{x}^{(i,2)}, j^{(i,2)}, y^{(i,2)})$. For person $i$ this indexed estimation algorithm outputs 
    \[
    \hat{s}_{j_i} = \hat{f}_i(\tilde{x}^{(i,2)}_{j_i})\tilde{x}^{(i,2)}_{j_i}.
    \]

    The next step is to compute the error of $\hat{s}_{j_i}$. In the proof of \Cref{thm:redclasstoest} we showed that for person $i$
    \[
    \min_{f_i}\pr{(\tilde{x}^{(i,2)},j^{(i,2)},y^{(i,2)})\sim \dist_{\text{class}}^{(i)}}{ f_i(\tilde{x}^{(i,2)}_{j_i}) \neq y^{(i,2)}} = \ind\{\text{sign}(p_{j_i}) = 1\} \frac{1-p_{j_i}}{2} + \ind\{\text{sign}(p_{j_i}) = -1\} \frac{1+p_{j_i}}{2}.
    \]
    The misclassification error of $\hat{f}_i$ is 
    \begin{align*}
        \pr{(\tilde{x}^{(i,2)},j^{(i,2)},y^{(i,2)})\sim \dist_{class}^{(i)}}{ \hat{f}_i(\tilde{x}^{(i,2)}_{j_i}) \neq y^{(i,2)}}  =  {}&\pr{(\tilde{x}^{(i,2)},j^{(i,2)},y^{(i,2)})\sim \dist_{class}^{(i)}}{ \hat{f}_i(\tilde{x}^{(i,2)}_{j_i})\tilde{x}^{(i,2)}_{j_i} \neq \tilde{x}^{(i,2)}_{j_i}y^{(i,2)}}\\
        = {} &\pr{(\tilde{x}^{(i,2)},j^{(i,2)},y^{(i,2)})\sim \dist_{class}^{(i)}}{ \hat{s}_i \neq w^{(i,2)}_{j_i}}\\
         = {}&\pr{(\tilde{x}^{(i,2)},j^{(i,2)},y^{(i,2)})\sim \dist_{class}^{(i)}}{ \hat{s}_i \neq 1|w^{(i,2)}_{j_i}=1} \pr{}{w^{(i,2)}_{j_i}=1}\\
        &+\pr{(\tilde{x}^{(i,2)},j^{(i,2)},y^{(i,2)})\sim \dist_{class}^{(i)}}{ \hat{s}_i \neq -1|w^{(i,2)}_{j_i}=-1} \pr{}{w^{(i,2)}_{j_i}=-1}.
    \end{align*}
    We can see that $\tilde{x}^{(i,2)}_{j_i}$ is independent of $w^{(i,2)}_{j_i}$. Thus, $\hat{s}_{j_i}$ is independent of  $w^{(i,2)}_{j_i}$ and we have that 
\begin{align*}
     \pr{(\tilde{x}^{(i,2)},j^{(i,2)},y^{(i,2)})\sim \dist_{class}^{(i)}}{ \hat{f}_i(\tilde{x}^{(i,2)}_{j_i}) \neq y^{(i,2)}} & = {}\pr{(\tilde{x}^{(i,2)},j^{(i,2)},y^{(i,2)})\sim \dist_{class}^{(i)}}{ \hat{s}_i \neq 1} \frac{1+p_{j_i}}{2}+\pr{(\tilde{x}^{(i,2)},j^{(i,2)},y^{(i,2)})\sim \dist_{class}^{(i)}}{ \hat{s}_i \neq -1} \frac{1-p_{j_i}}{2}\\
     & = \ex{(\tilde{x}^{(i,2)},j^{(i,2)},y^{(i,2)})\sim \dist_{class}^{(i)}}{ \ind\{\hat{s}_i \neq 1\} \frac{1+p_{j_i}}{2}+ \ind\{\hat{s}_i \neq -1\} \frac{1-p_{j_i}}{2}}
\end{align*}
Combining the these steps, we get that the difference of the misclassification error between $\hat{f}_i$ and the optimal classifier is 
\begin{align*}
    &\pr{(\tilde{x}^{(i,2)},j^{(i,2)},y^{(i,2)})\sim \dist_{class}^{(i)}}{ \hat{f}_i(\tilde{x}^{(i,2)}_{j_i}) \neq y^{(i,2)}} - \min_{f_i}\pr{(\tilde{x}^{(i,2)},j^{(i,2)},y^{(i,2)})\sim \dist_{\class}^{(i)}}{ f_i(\tilde{x}^{(i,2)}_{j_i}) \neq y^{(i,2)}} \\& = \ex{(\tilde{x}^{(i,2)},j^{(i,2)},y^{(i,2)})\sim \dist_{class}^{(i)}}{ \ind\{\hat{s}_i \neq 1\} \frac{1+p_{j_i}}{2}+ \ind\{\hat{s}_i \neq -1\} \frac{1-p_{j_i}}{2} - \ind\{\text{sign}(p_{j_i}) = 1\} \frac{1-p_{j_i}}{2} - \ind\{\text{sign}(p_{j_i}) = -1\} \frac{1+p_{j_i}}{2}}\\
    & = \ex{(\tilde{x}^{(i,2)},j^{(i,2)},y^{(i,2)})\sim \dist_{class}^{(i)}}{ \ind\{\hat{s}_i \neq \text{sign}(p_{j_i}) |p_{j_i}|\} }.
\end{align*}
The average error of the sign estimation algorithm, in expectation over the samples, the randomness of the input/output transformations and the randomness of algorithm $\alg$ is 
\begin{align*}
    &\frac{1}{\users} \sum_{i \in [\users]} \ex{}{\ind\{\text{sign}(p_{j_i}) \neq \hat{s}_{j_i}\}|p_{j_i}|}\\ 
    &= \frac{1}{\users} \sum_{i \in [\users]} \ex{}{\ex{(\tilde{x}^{(i,2)}, j^{(i,2)}, y^{(i,2)})\sim \dist_{\text{class}}^{(i)}}{\ind\{\text{sign}(p_{j_i}) \neq \hat{s}_{j_i}\}|p_{j_i}|}}\\
    & = \frac{1}{\users} \sum_{i \in [\users]} \ex{}{\pr{(\tilde{x}^{(i,2)},j^{(i,2)},y^{(i,2)})\sim \dist_{class}^{(i)}}{ \hat{f}_i(\tilde{x}^{(i,2)}_{j_i}) \neq y^{(i,2)}} - \min_{f_i}\pr{(\tilde{x}^{(i,2)},j^{(i,2)},y^{(i,2)})\sim \dist_{\text{class}}^{(i)}}{ f_i(\tilde{x}^{(i,2)}_{j_i}) \neq y^{(i,2)}}}\\
    & \leq \alpha.
\end{align*}
\end{proof}
\subsection{Upper Bounds}

We prove upper bounds on the indexed classification problem by reducing to and from a related learning problem called indexed sign estimation, encoding the sign estimation problem into a mean estimation problem, and finally applying the upper bounds from Section~\ref{sec:estimation} for indexed mean estimation.

We start by proving an upper bound for multitask learning without privacy.

\newcommand{\lemClassNpUB}{ For any $\datadim \in \mathbb{N}$ and $\users \in \mathbf{N}$, and loss function $\ell_\text{class}$, there exists an algorithm that multitask learns $\distfam_{\text{class}, \datadim, \users}$ with error \(\frac{1}{\sqrt{\users}}\) with $\users$ tasks and $1$ sample per task.}
\begin{lem}[Non-private upper bound]
\lemClassNpUB
\label{lem:classnpub}
\end{lem}
\begin{proof} 
    Fix any $\datadim \in \N$, $\users \in \N$. 
    By \Cref{prop:meanestnpub} there is an algorithm $\alg$ that multitask learns $\distfam_{\text{est}, \datadim, \users}$ with error
    \[
    \frac{1}{\users}\sum_{i \in [\users]} \ex{}{(\hat{p}_{j_i} -p_{j_i})^2} \leq \frac{1}{\users},
    \]
    where $\hat{p}_{j_i}$ is the output of algorithm $\alg$ to person $i$, with $\users$ tasks and $1$ sample per task. 
    We can use $\alg$ to estimate the signs of the $p_{j_i}$s by setting $\hat{s}_{j_i} = \text{sign}(\hat{p}_{j_i})$. In this case the error for indexed sign estimation is 
    \begin{align*}
        \frac{1}{\users}\sum_{i \in [\users]}\ex{}{\ind\{\hat{s}_{j_i} \neq \text{sign}(p_{j_i}) \}|p_{j_i}|}&\leq \frac{1}{\users}\sum_{i \in [\users]} \ex{}{\ind\{\hat{s}_{j_i} \neq \text{sign}(p_{j_i}) \}|\hat{p}_{j_i} - p_{j_i}|}\\
        &\leq \frac{1}{\users}\sum_{i \in [\users]} \ex{}{\ind\{\hat{s}_{j_i} \neq \text{sign}(p_{j_i}) \}|\hat{p}_{j_i} - p_{j_i}|+ \ind\{\hat{s}_{j_i} = \text{sign}(p_{j_i}) \}|\hat{p}_{j_i} - p_{j_i}|}\\
        & = \frac{1}{\users}\sum_{i \in [\users]} \ex{}{|\hat{p}_{j_i} - p_{j_i}|}\\
        & \leq \frac{1}{\users}\ex{}{\sqrt{\sum_{i \in [\users]}(\hat{p}_{j_i} - p_{j_i})^2} \sqrt{\users}}\\
        &\leq  \sqrt{\frac{1}{\users}\sum_{i \in [\users]}\ex{}{(\hat{p}_{j_i} - p_{j_i})^2}} \\
        &\leq \frac{1}{\sqrt{\users}}.
    \end{align*}
    By applying the reduction of \Cref{thm:redclasstoest} we obtain that for all $(\dist_1, \ldots, \dist_\users) \in \distfam_{\text{class}, \datadim, \users}$
    \[
    \frac{1}{\users} \sum_{i \in [\users]} \ex{}{\left( \pr{(x,j,y) \sim \dist_i}{\hat{f}_i(x_{j_i}) \neq y} - \min_{f_i} \pr{(x,j,y) \sim \dist_i}{f_i(x_{j_i}) \neq y}\right)} \leq \frac{1}{\users} \sum_{i \in [\users]} \ex{}{\ind\{\hat{s}_{j_i} \neq \text{sign}(p_{j_i}) \}|p_{j_i}|} 
        \leq \frac{1}{\sqrt{\users}}.
    \]
\end{proof}
For one-out-of-$\users$ privacy the multitask learning error has an additional factor of $O(1/\users)$ compared to the non-private solution.

\newcommand{\thmClassoootUB}{ For any parameters $\rho \geq 0$, $\datadim \in \mathbb{N}$, $\users \in \mathbb{N}$ and for loss function $\ell_\text{class}$, there exists an algorithm that multitask learns $\distfam_{\text{class}, \datadim, \users}$ with error \[\frac{\sqrt{2}}{t\sqrt{\rho}} + \frac{1}{\sqrt{t}}\] with $\users$ tasks and $1$ sample per task and satisfies $\rho$-$1$-out-of-$\users$-zCDP. Moreover, it is implied that for any parameters $\epsilon > 0$, $\delta \in (0,1)$, $\datadim \in \mathbb{N}$, $\users \in \mathbb{N}$ and for loss function $\ell_\text{class}$, there exists an algorithm that multitask learns $\distfam_{\text{class}, \datadim, \users}$ with error $\Theta( \min(\epsilon, {\epsilon^2}/{\log(1/\delta)})^{-1/2})\cdot\frac{1}{t} + \frac{1}{\sqrt{t}}$ with $\users$ tasks and $1$ sample per task, and satisfies $(\eps, \delta)$-1-out-of-$\users$-DP.
}
\begin{thm}[1-out-of-$t$ upper bound]
\thmClassoootUB
\label{thm:classoootub}
\end{thm}
\begin{proof}

Fix any $\datadim \in \N$, $\users \in \N$, and $\rho \geq 0$.
By \Cref{thm:meanestoootub} there exists a $\rho$-$1$-out-of-$\users$-zCDP algorithm $\alg$ that multitask learns $\distfam_{\text{est}, \datadim, \users}$ with error
\[
    \frac{1}{\users}\sum_{i \in [\users]} \ex{}{(\hat{p}_{j_i} -p_{j_i})^2} \leq \frac{2}{\rho \users^2}+ \frac{1}{\users},
    \]
with $\users$ tasks and $1$ sample per task.
In this algorithm the computes a mean estimate $\hat{p}_{j_i}$ that they send to person $i$. For the classification learning algorithm the curator can use $\hat{p}_{j_i}$ to construct an estimate of the sign of $p_{j_i}$ by setting $\hat{s}_{j_i} = \text{sign}(\hat{p}_{j_i})$ before sending anything to the people. Then, using the reduction of \Cref{thm:redesttoclass} we get a function that the server can send to person $i$, by setting 
\[
\hat{f}_i(x,j) = \hat{s}_{j_i} x_{j_i}.
\]
As $\hat{f}_i$ is a post-processing of $\hat{p}_{j_i}$ that only uses $j_i$, which is already in the dataset of person $i$, the output of this algorithm remains $\rho$-$1$-out-of-$\users$-zCDP. By \Cref{thm:redclasstoest} we obtain that the error of of $\hat{f}_1, \ldots, \hat{f}_\users$ is 
\begin{align*}
     \frac{1}{\users} \sum_{i \in [\users]} \ex{}{\left( \pr{(x,j,y) \sim \dist_i}{\hat{f}_i(x_{j_i}) \neq y} - \min_{f_i} \pr{(x,j,y) \sim \dist_i}{f_i(x_{j_i}) \neq y}\right)} &\leq \frac{1}{\users} \sum_{i \in [\users]} \ex{}{\ind\{\hat{s}_{j_i} \neq \text{sign}(p_{j_i}) \}|p_{j_i}|} \\
    &\leq \sqrt{\frac{1}{\users}\sum_{i \in [\users]}\ex{}{(\hat{p}_{j_i} - p_{j_i})^2}}\\
    & \leq \sqrt{\frac{2}{\rho \users^2}+\frac{1}{\users}}\\
    & \leq \sqrt{\frac{2}{\rho \users^2}}+\sqrt{\frac{1}{\users}}.
\end{align*}
\end{proof}
When the multitask learning algorithm satisfies the stronger privacy guarantee of JDP the additive factor becomes $O(1/\sqrt{\users})$.

\newcommand{\thmClassJDPUB}{ For any parameters $\rho \geq 0$, $\datadim \in \mathbb{N}$, $\users \in \mathbb{N}$, and for loss function $\ell_\text{class}$, there exists an algorithm that multitask learns $\distfam_{\text{class}, \datadim, \users}$ with error \[O\left( \frac{1}{\sqrt{\rho t}}\right) + \frac{1}{\sqrt{t}}\] with $\users$ tasks and $1$ sample per task, and satisfies $\rho$-JCDP.
Moreover, it is implied that for any $\eps > 0$, $\delta \in (0,1)$, $\datadim \in \mathbb{N}$, $\users \in \mathbb{N}$, and for loss function $\ell_\text{class}$, there exists an algorithm that multitask learns $\distfam_{\text{class}, \datadim, \users}$ with error $\Theta( \min(\epsilon, {\epsilon^2}/{\log(1/\delta)})^{-1/2})\cdot\frac{1}{\sqrt{t}}  + \frac{1}{\sqrt{t}}$ with $\users$ tasks and $1$ sample per task, and satisfies $(\epsilon, \delta)$-JDP.
}

\begin{thm}[JDP upper bound]
    \thmClassJDPUB
    \label{thm:classjdpub}
\end{thm}
\begin{proof}
    Fix any $\datadim \in \N$, $\users \in \N$ and $\rho \geq 0$.
    Similarly to the proof above we can base our algorithm on the algorithm of \Cref{thm:meanestjdpub}. Our classification algorithm takes the mean estimate $\hat{p}_{j_i}$ for person $i$ and produces 
    \[
    \hat{f}_i(x_{j_i}) = \hat{s}_{j_i} x_{j_i},
    \]
    with error 
    \[
     \frac{1}{\users} \sum_{i \in [\users]} \ex{}{\left( \pr{(x,j,y) \sim \dist_i}{\hat{f}_i(x_{j_i}) \neq y} - \min_{f_i} \pr{(x,j,y) \sim \dist_i}{f_i(x_{j_i}) \neq y}\right)} \leq\sqrt{\frac{1}{\users}\sum_{i \in [\users]}\ex{}{(\hat{p}_{j_i} - p_{j_i})^2}} \leq \frac{\sqrt{2(\users-1)}}{t \sqrt{\rho}} + \frac{1}{\sqrt{\users}}.
    \]
    The new algorithm also satisfies $\rho$-JCDP because the curator just post-processes the output of the $\rho$-JCDP mean estimation multitask learning algorithm using only $j_i$ for person $i$, which is their personal information.
\end{proof}
When the people participating in multitask learning interact through a DP billboard algorithm, the additional factor of the error is $O(\sqrt{\datadim}/\users)$.

\newcommand{\thmClassBBUB}{For any parameters $\rho \geq 0$, $\datadim \in \mathbb{N}$, $\users \in \mathbb{N}$, and for loss function $\ell_\text{class}$, there exists a billboard algorithm that multitask learns $\distfam_{\text{class},\datadim, \users}$ with error \[\frac{\sqrt{2d}}{t \sqrt{\rho} } + \frac{1}{\sqrt{t}}\] with $\users$ tasks and $1$ sample per task, and satisfies $\rho$-zCDP.
Moreover, it is implied that for any parameters $\eps > 0$, $\delta \in (0,1)$, $\datadim \in \mathbb{N}$, $\users \in \mathbb{N}$, and for loss function $\ell_\text{class}$,
there exists a billboard algorithm that multitask learns $\distfam_{\text{class}, \datadim, \users}$ with error $\Theta( \min(\epsilon, {\epsilon^2}/{\log(1/\delta)})^{-1/2})\cdot\frac{\sqrt{d}}{t}  + \frac{1}{\sqrt{t}}$ with $\users$ tasks and $1$ sample per task, and satisfies $(\epsilon, \delta)$-DP.}

\begin{thm}[Billboard upper bound]
    \thmClassBBUB
    \label{thm:classbbub}
\end{thm}
\begin{proof}
    By \Cref{thm:meanestbbub} there exists a billboard algorithm $(\alg_{\text{BB}}, \alg_{\text{pers}})$ that multitask learns $\distfam_{\text{est}, \datadim, \users}$ with $\users$ tasks and $1$ sample per task, and $\alg_{\text{BB}}$ is $\rho$-zCDP. Our classification algorithm keeps the billboard algorithm as it is and modifies the personalization part. In particular, person $i$ takes their mean estimate $\hat{p}_{j_i}$ and produces 
    \[
    \hat{f}_i(x_{j_i}) = \hat{s}_{j_i} x_{j_i},
    \]
    where $\hat{s}_{j_i} = \text{sign}(\hat{p}_{j_i})$.
    By applying the reduction of \ref{thm:redesttoclass} we get error 
    \[
     \frac{1}{\users} \sum_{i \in [\users]} \ex{}{\left( \pr{(x,j,y) \sim \dist_i}{\hat{f}_i(x_{j_i}) \neq y} - \min_{f_i} \pr{(x,j,y) \sim \dist_i}{f_i(x_{j_i}) \neq y}\right)} \leq\sqrt{\frac{1}{\users}\sum_{i \in [\users]}\ex{}{(\hat{p}_{j_i} - p_{j_i})^2}} \leq \frac{\sqrt{2d}}{t \sqrt{\rho}} + \frac{1}{\sqrt{\users}}.
    \]
    As we mentioned, the billboard part of the new algorithm remains the same. Thus, it still satisfies $\rho$-zCDP.
\end{proof}
Finally, \Cref{thm:classbbub} and the reduction of metalearning to multitask learning in \Cref{thm:metatomulti} imply that there exists a mealearning algorithm whose metalearning process is differentially private.

\begin{cor}[Metalearning upper bound]
For any $\rho \geq 0$, $\datadim \in \mathbb{N}$, $\users \in \mathbb{N}$, and for loss function $\ell_\text{class}$, there exists a pair of algorithms $(\alg_{\text{meta}}, \alg_{\text{pers}})$ that metalearn a metadistribution $\metadist$ over $\distfam_{\text{class},\datadim, \users+1}$ with error \[\Theta( e^\epsilon\cdot \min(\epsilon, {\epsilon^2}/{\log(1/\delta)})^{-1/2})\cdot\frac{\sqrt{d}}{t}  + \frac{e^\epsilon}{\sqrt{t}} + \delta\] with $\users$ training tasks, $1$ sample per training task and a test task with $1$ personalization sample
 and $\alg_{\text{meta}}$ satisfies $(\epsilon, \delta)$-DP.
\end{cor} 
\label{sec:classub}

\subsection{Our Fingerprinting Lemma}
Our error lower bounds for indexed classification are based on a generalization of the fingerprinting lemma by \cite{BunUV14}. Our lemma differs in two ways; the mean $p$ is drawn uniformly at random from an interval $[-\alpha, \alpha]$ for $\alpha \in (0,1]$ (instead of $[-1,1]$) and the error function is $2(|p|-f(x))p = 4|p| \ind \{f(x) \neq \text{sign}(p)\}$ (instead of the mean-squared-error). The key idea is that for a fixed $\alpha$ a function $f$ cannot be bad at predicting the sign of the mean $p$ and highly correlated with the sum of samples drawn from a distribution with mean $p$ at the same time. For the proof see \Cref{sec:fp-proof}.

\newcommand{\lemOurFpL}{Let $f: \{\pm 1\}^\users \to [\pm1]$ and $\alpha \in (0,1]$. If $p \in [-\alpha, +\alpha]$ is sampled uniformly at random and $ x_1,\ldots, x_\users \in \{\pm1\}^\users$ are sampled independently with mean $p$, then
     \[
     \ex{\substack{p,\\ x_1,\ldots, x_\users }}{r(\alpha, p)f(x) \sum_{i \in [\users]}  (x_i-p) + 2(|p|-f(x)p)} = \alpha,
     \]
     where 
     \[
     r(\alpha,p) = \begin{cases}
         \frac{\alpha^2-p^2}{1-p^2}, & \text{if }\alpha^2 \neq 1,\\
         1, &\text{otherwise}.
     \end{cases}
     \]}
 \begin{lem} [Our Fingerprinting Lemma]
    \lemOurFpL
     \label{lem:our-fp-lemma}
 \end{lem}

\subsection{Lower Bounds}

We also prove lower bounds for the error of indexed classification that separate the interaction and privacy models in the multitask learning setting. We initially prove lower bounds for indexed sign estimation using our fingerprinting lemma. Finally, we prove the lower bounds for indexed classification by using a reduction of indexed sign estimation to indexed classification. The proofs of the theorems in this section are provided in \Cref{sec:indClass-lb-proofs}.

\newcommand{\thmClassJDPLB}{Fix  $\alpha \in (0,\frac{1}{16})$, $\users \in \mathbb{N}$, $d\ge ct^2$ for a sufficiently large constant $c$, $\eps > 0$ and $\delta\in (0,\frac{1}{2t})$. Let $\alg$ be an algorithm that multitask learns  $\distfam_{\text{class},\datadim, \users}$ with error $\alpha$, for loss function $\ell_\text{class}$, with $\users$ tasks and $1$ sample per task, and satisfies $(\eps, \delta)$-JDP.
  Then, \[\alpha \ge \Omega\left(\min\left\{\frac{1}{ \eps\sqrt{\users}},1\right\}\right).\]}
\begin{thm}[classification JDP lower bound]
    \thmClassJDPLB
    \label{thm:classjdplb}
\end{thm}

\newcommand{\thmClassMetaLB}{Fix  $\alpha \in (0,\frac{1}{8})$, $\users \in \mathbb{N}$, $d \in \mathbb{N}$, $\eps \in (0,1]$ and $\delta\in (0,\frac{1}{2t})$. 
Let $\alg = (\alg_{\text{meta}}, \alg_{\text{pers}})$ be a pair of algorithms that metalearn a distribution $\metadist$ over $\distfam_{\text{class},\datadim, \users+1}$ with error $\alpha$, for loss function $\ell_\text{class}$, using $\users$ training tasks, $1$ sample per training task and a test task with $1$ personalization sample, and $\alg_{\text{meta}}$ satisfies $(\eps, \delta)$-DP.
Then, \[\alpha \ge \Omega\left(\min\left\{ \frac{\sqrt{\datadim}}{\eps \users},1\right\}\right).\]}

\begin{thm}[Metalearning lower bound]
   \thmClassMetaLB
   \label{thm:classmetalb}
\end{thm}

\newcommand{\thmClassBBLB}{ Fix $\alpha \in (0,\frac{1}{8})$, $\users \in \mathbb{N}$, $d\ge \frac{\eps^2 t}{4}$, $\eps \in (0,1]$ and $\delta\in (0, \frac{1}{32^2t})$. 
Let $\alg$ be a billboard algorithm that multitask learns $\distfam_{\text{class}, \datadim, \users}$ with error $\alpha$, for loss function $\ell_\text{class}$, with $\users$ tasks and $1$ sample per task 
and satisfies $(\eps, \delta)$-DP.
Then, \[\alpha \ge \Omega\left(\min\left\{\frac{\sqrt{d}}{\epsilon t},1\right\}\right).\]
}

\begin{thm}[Billboard lower bound]
    \thmClassBBLB
    \label{thm:classbblb}
\end{thm}

\addcontentsline{toc}{section}{References}
\bibliographystyle{alpha-no-nonsense}
\bibliography{refs}

\newcommand{\etalchar}[1]{$^{#1}$}
\begin{thebibliography}{KMA{\etalchar{+}}19}

\bibitem[ABB{\etalchar{+}}24]{AliakbarpourBBSSU24}
Maryam Aliakbarpour, Konstantina Bairaktari, Gavin Brown, Adam Smith, Nathan Srebro, and Jonathan Ullman.
\newblock Metalearning with very few samples per task.
\newblock In {\em Conference on Learning Theory}, COLT '24. PMLR, 2024.

\bibitem[BHPQ17]{BlumHPQ17}
Avrim Blum, Nika Haghtalab, Ariel~D Procaccia, and Mingda Qiao.
\newblock Collaborative pac learning.
\newblock {\em Advances in Neural Information Processing Systems}, 30, 2017.

\bibitem[BNS{\etalchar{+}}16]{BassilyNSSSU16}
Raef Bassily, Kobbi Nissim, Adam Smith, Thomas Steinke, Uri Stemmer, and Jonathan Ullman.
\newblock Algorithmic stability for adaptive data analysis.
\newblock In {\em ACM Symposium on the Theory of Computing}, STOC '16, 2016.

\bibitem[BS16]{BunS16}
Mark Bun and Thomas Steinke.
\newblock Concentrated differential privacy: Simplifications, extensions, and lower bounds.
\newblock In {\em Theory of Cryptography Conference}, TCC '16, 2016.
\newblock \url{https://arxiv.org/abs/1605.02065}.

\bibitem[BSU17]{BunSU17}
Mark Bun, Thomas Steinke, and Jonathan Ullman.
\newblock Make up your mind: The price of online queries in differential privacy.
\newblock In {\em Proceedings of the 28th Annual ACM-SIAM Symposium on Discrete Algorithms}, SODA '17, 2017. SIAM.

\bibitem[BUV14]{BunUV14}
Mark Bun, Jonathan Ullman, and Salil Vadhan.
\newblock Fingerprinting codes and the price of approximate differential privacy.
\newblock In {\em ACM Symposium on the Theory of Computing}, STOC '14, 2014.
\newblock \url{https://arxiv.org/abs/1311.3158}.

\bibitem[BUV18]{BunUV14j}
Mark Bun, Jonathan Ullman, and Salil Vadhan.
\newblock Fingerprinting codes and the price of approximate differential privacy.
\newblock {\em SIAM Journal on Computing}, 47(5):1888--1938, 2018.

\bibitem[DFH{\etalchar{+}}14]{DworkFHPRR14}
Cynthia Dwork, Vitaly Feldman, Moritz Hardt, Toniann Pitassi, Omer Reingold, and Aaron Roth.
\newblock Preserving statistical validity in adaptive data analysis.
\newblock {\em arXiv preprint arXiv:1411.2664}, 2014.

\bibitem[DMNS06]{DworkMNS06}
Cynthia Dwork, Frank McSherry, Kobbi Nissim, and Adam Smith.
\newblock Calibrating noise to sensitivity in private data analysis.
\newblock In {\em Conference on Theory of Cryptography}, TCC '06, 2006.

\bibitem[DSS{\etalchar{+}}15]{DworkSSUV15}
Cynthia Dwork, Adam Smith, Thomas Steinke, Jonathan Ullman, and Salil Vadhan.
\newblock Robust traceability from trace amounts.
\newblock In {\em IEEE Symposium on Foundations of Computer Science}, FOCS '15, 2015.

\bibitem[GLM{\etalchar{+}}10]{GuptaLMRT10}
Anupam Gupta, Katrina Ligett, Frank McSherry, Aaron Roth, and Kunal Talwar.
\newblock Differentially private combinatorial optimization.
\newblock In {\em Proceedings of the Twenty-First Annual {ACM-SIAM} Symposium on Discrete Algorithms, {SODA} 2010, Austin, Texas, USA, January 17-19, 2010}, 2010.

\bibitem[HHR{\etalchar{+}}14]{HsuHRRW14}
Justin Hsu, Zhiyi Huang, Aaron Roth, Tim Roughgarden, and Zhiwei~Steven Wu.
\newblock Private matchings and allocations.
\newblock In {\em ACM symposium on Theory of computing}, 2014.

\bibitem[JLN{\etalchar{+}}20]{JungLNRSS20}
Christopher Jung, Katrina Ligett, Seth Neel, Aaron Roth, Saeed Sharifi-Malvajerdi, and Moshe Shenfeld.
\newblock A new analysis of differential privacy's generalization guarantees.
\newblock In {\em Innovations in Theoretical Computer Science (ITCS)}, 2020.

\bibitem[JRS{\etalchar{+}}21]{JainRSST21}
Prateek Jain, J.~Keith Rush, Adam Smith, Shuang Song, and Abhradeep Thakurta.
\newblock Differentially private model personalization.
\newblock In {\em Advances in Neural Information Processing Systems 33 (NeurIPS 2021)}, 2021.

\bibitem[KJS{\etalchar{+}}23]{KricheneJSSTZ2023multi}
Walid Krichene, Prateek Jain, Shuang Song, Mukund Sundararajan, Abhradeep~Guha Thakurta, and Li~Zhang.
\newblock Multi-task differential privacy under distribution skew.
\newblock In {\em International Conference on Machine Learning}. PMLR, 2023.

\bibitem[KMA{\etalchar{+}}19]{Kairouz+19}
Peter Kairouz, H~Brendan McMahan, Brendan Avent, Aur{\'e}lien Bellet, Mehdi Bennis, Arjun~Nitin Bhagoji, Keith Bonawitz, Zachary Charles, Graham Cormode, and Rachel Cummings.
\newblock Advances and open problems in federated learning.
\newblock {\em arXiv preprint arXiv:1912.04977}, 2019.

\bibitem[KMRR18]{KannanMRR18}
Sampath Kannan, Jamie Morgenstern, Ryan Rogers, and Aaron Roth.
\newblock Private pareto optimal exchange.
\newblock {\em ACM Transactions on Economics and Computation (TEAC)}, 6(3-4):1--25, 2018.

\bibitem[KPRU14]{KearnsPRU14}
Michael Kearns, Mallesh~M. Pai, Aaron Roth, and Jonathan Ullman.
\newblock Mechanism design in large games: incentives and privacy.
\newblock In {\em Proceedings of the 5th ACM Conference on Innovations in Theoretical Computer Science}, ITCS '14, 2014. ACM.

\bibitem[KU20]{Kamath2020primer}
Gautam Kamath and Jonathan Ullman.
\newblock A primer on private statistics.
\newblock {\em arXiv preprint arXiv:2005.00010}, 2020.

\bibitem[LNR{\etalchar{+}}17]{LigettNRWW17}
Katrina Ligett, Seth Neel, Aaron Roth, Bo~Waggoner, and Steven~Z Wu.
\newblock Accuracy first: Selecting a differential privacy level for accuracy constrained erm.
\newblock In {\em Advances in Neural Information Processing Systems 30}, NIPS '17, 2017.

\bibitem[MM09]{McSherryM09}
Frank McSherry and Ilya Mironov.
\newblock Differentially private recommender systems: Building privacy into the netflix prize contenders.
\newblock In {\em ACM SIGKDD international conference on Knowledge discovery and data mining}, 2009.

\bibitem[PTU24]{PeterTU24}
Naty Peter, Eliad Tsfadia, and Jonathan Ullman.
\newblock Smooth lower bounds for differentially private algorithms via padding-and-permuting fingerprinting codes.
\newblock In {\em Conference on Learning Theory}, COLT '24. PMLR, 2024.

\bibitem[Ste16]{steinke2016upper}
Thomas~Alexander Steinke.
\newblock {\em Upper and Lower Bounds for Privacy and Adaptivity in Algorithmic Data Analysis}.
\newblock PhD thesis, 2016.

\bibitem[Ull13]{Ullman13}
Jonathan Ullman.
\newblock Answering $n^{2+o(1)}$ counting queries with differential privacy is hard.
\newblock In {\em ACM Symposium on the Theory of Computing}, STOC '13, 2013.

\end{thebibliography}

\newpage

\appendix

\section{Proofs from \Cref{sec:estimation}}
\label{sec:estimation_proofs}
\begin{lem}[Birthday Paradox]
\label{lem:birthdayp}
    Suppose we toss $\users$ balls uniformly at random into $\datadim$ bins. If $\datadim \geq c \users^2$ for a sufficiently large constant $c$, the probability of any two balls falling into the same bin is at most $\frac{1}{2}$. 
\end{lem}

\begin{customthm}{\ref*{thm:mean-JDP-lb}}[Restated]
    \thmMeanEstJDPLB
\end{customthm}
\begin{proof}
For the proof of \Cref{thm:mean-JDP-lb} we follow a series of steps similar to that presented in \cite{Kamath2020primer}, which itself closely follows the presentation in \cite{steinke2016upper}. 

We have a JDP algorithm on $\users$ tasks/people and $\datadim$ dimensions. Each person $i \in [\users]$, has a sample $(x^{(i)}, j^{(i)}) \in \{\pm 1\} \times [\datadim]$ and gives it as input to the algorithm. Let $S = \{(x^{(i)}, j^{(i)})\}_{i \in [\users]}$ denote this dataset. We will construct a hard distribution over the family of distributions $\distfam_{\datadim,\users}$ that the samples are drawn from. We draw a vector $\mathbf{p} \in [-1,+1]^\datadim$ and a vector $\mathbf{j} \in [\datadim]^\users$, both uniformly at random. The datapoint $x^{(i)}$ of person $i$ is drawn from the product distribution with mean $\textbf{p}$ and the index $j^{(i)}$ is deterministically $j_i$. Let the expected error of $\alg$ be 
\[
  \ex{\mathbf{p}, \mathbf{j}}{\sum_{i \in [\users]}\ex{\substack{S, \alg}}{\frac{1}{4}(\hat{p}_{j_i} - p_{j_i})^2}} \leq \alpha.
\]
Our goal now is to prove a lower bound on $\alpha$.

We will analyze the case where $\mathbf{j}$ has no duplicated indices. So for now we assume that we have a fixed $\mathbf{j}$ with no repeats. For person $i$ we define two test statistics 
\begin{align*}
    T_i & \defeq \sum_{\ell \in [\users] \setminus \{i\}} (\hat{p}_{j_\ell} - p_{j_\ell})(x^{(i)}_{j_\ell} - p_{j_\ell})\\
    T_i' & \defeq \sum_{\ell \in [\users] \setminus \{i\}} (\hat{p}_{j_\ell}' - p_{j_\ell})(x^{(i)}_{j_\ell} - p_{j_\ell})
\end{align*}
where $\hat{p}_{j_\ell}'$ denotes the output of algorithm $\alg$ to person $\ell$ for input $S_{\sim i}$ where $\left(x^{(i)}, j^{(i)}\right)$ has been replaced with a fresh draw $\left({x^{(i)}}', {j^{(i)}}'\right)$ from $P_i$. Since the distribution over the indices is deterministic, $j^{(i)} = {j^{(i)}}' = j_i$. 

Now, we will use the privacy guarantee of JDP to provide an upper bound for $\ex{\substack{\mathbf{p},S,\alg}}{\sum_{i \in [\users]}T_i}$. Since $\alg$ is $(\eps, \delta)$-JDP for any $\eps, \delta$ in $(0,1)$, we have that $\{\hat{p}_{j_\ell}\}_{\ell \in [\users]\setminus i}$ is $(\eps, \delta)$-DP with respect to $i$'s dataset. As a result, the following inequality holds 
\begin{align}
\ex{}{T_i} \leq \ex{}{T_i'} + 2\eps \sqrt{\var{}{T_i'}} + 2\delta \|T_i'\|_{\infty}.
\label{ineq:dp-ub}
\end{align} 

$T_i'$ is a sum of $\users-1$ values in $[-4,4]$, so the bound $\|T_i'\|_{\infty} \leq 4(\users-1) \leq 4\users$ holds. For fixed $\mathbf{p}$, $\mathbf{j}$, and $\ell$, $\hat{p}_{j_\ell}'$ is independent of $x_{j_\ell}^{(i)}$ and $\ex{}{x_{j_\ell}^{(i)}} = p_{j_\ell}$. Thus, for any $\ell \in [\users] \setminus \{i\}$
\begin{align*}
& \ex{}{ ({\hat{p}_{j_\ell}}' - p_{j_\ell})(x^{(i)}_{j_\ell} - p_{j_\ell})} \\
&= \ex{\mathbf{p}}{ \ex{\substack{S_{\sim i},  \alg}}{({\hat{p}_{j_\ell}}' - p_{j_\ell})}\ex{\substack{S,  \alg}}{(x_{j_\ell}^{(i)} - p_{j_\ell})}} = 0.
\end{align*}
This means that $\ex{}{T_i'} = 0$. We apply the same observation, that $\ex{}{x_{j_\ell}^{(i)}} = p_{j_\ell}$, and that every coordinate is independent to show that the cross terms in the variance of $T_i'$ cancel out, leaving us with 
\[
\var{}{T_i'} = \ex{}{(T_i')^2} = \ex{}{\sum_{\ell \in [\users]\setminus{i}}({\hat{p}_{j_\ell}}' - p_{j_\ell})^2(x^{(i)}_{j_\ell} - p_{j_\ell})^2}. 
\]
Since $(x^{(i)}_{j_\ell} - p_{j_\ell})^2$ is at most $4$, we have the upper bound 
\begin{align*}
    \var{}{T_i'}& = \ex{}{\sum_{\ell \in [\users]\setminus{i}}({\hat{p}_{j_\ell}}' - p_{j_\ell})^2(x^{(i)}_{j_\ell} - p_{j_\ell})^2}\\
    & \leq 4\ex{}{\sum_{\ell \in [\users]\setminus{i}}({\hat{p}_{j_\ell}}' - p_{j_\ell})^2}\\
    & \leq 4\ex{}{\sum_{\ell \in [\users]}({\hat{p}_{j_\ell}}' - p_{j_\ell})^2}.
\end{align*}
Plugging these into inequality \ref{ineq:dp-ub} and summing up over all the $\users$ people, we obtain that 
\begin{align*}
    \ex{}{\sum_{i\in[\users]}T_i} \leq 4\eps \users \sqrt{\ex{}{\sum_{\ell \in [\users]}(\hat{p}_{j_\ell} - p_{j_\ell})^2}} + 8\delta \users^2.
\end{align*}

The next step is to show that accuracy implies a lower bound for $\ex{\substack{\mathbf{p},S,\alg}}{\sum_{i \in [\users]}T_i}$. We notice that in the sum of $T_i$ we can rearrange the terms.
\begin{align*}
    \sum_{i \in [\users]} T_i &= \sum_{i \in [\users]} \sum_{\ell \in [\users] \setminus \{i\}} (\hat{p}_{j_\ell} - p_{j_\ell})(x^{(i)}_{j_\ell} - p_{j_\ell})\\
    & = \sum_{\ell \in [\users]} \sum_{i \in [\users] \setminus \{\ell\}} (\hat{p}_{j_\ell} - p_{j_\ell})(x^{(i)}_{j_\ell} - p_{j_\ell}).\\
    & = \sum_{\ell \in [\users]}  (\hat{p}_{j_\ell} - p_{j_\ell})\sum_{i \in [\users] \setminus \{\ell\}}(x^{(i)}_{j_\ell} - p_{j_\ell}).\\
\end{align*}
We can now apply \Cref{lem:fp-lemma} to each coordinate $j_\ell$.
\begin{align*}
    \ex{}{\sum_{i\in [\users]}T_i} & = \sum_{\ell \in [\users]} \ex{}{(\hat{p}_{j_\ell} - p_{j_\ell})\sum_{i \in [\users] \setminus \{\ell\}}(x^{(i)}_{j_\ell} - p_{j_\ell})}\\
    & \geq \sum_{\ell \in [\users]}\left( \frac{1}{3} - \ex{}{(\hat{p}_{j_\ell} - p_{j_\ell})^2}\right)\\
    & = \frac{\users}{3} - \ex{}{ \sum_{\ell \in [\users]}(\hat{p}_{j_\ell} - p_{j_\ell})^2}
\end{align*}
For a vector of indices $\mathbf{j}$ let $\alpha_{\mathbf{j}}$ be the error 
\[
\alpha_\mathbf{j} = \ex{\mathbf{p}, S, \alg}{ \sum_{\ell \in [\users]}(\hat{p}_{j_\ell} - p_{j_\ell})^2}.
\]
Combining the two inequalities for $ \ex{}{\sum_{i\in [\users]}T_i}$ we have shown that when $J$ has no duplicates 
\[
\frac{\users}{3} - \alpha_\mathbf{j} \leq \ex{}{\sum_{i\in [\users]}T_i}  \leq 4\eps \users \alpha_J + 8\delta \users^2.
\]
If $\alpha_\mathbf{j} \leq \frac{t}{6}$, then $8\delta \users^2 <\frac{\users}{12}$ because we have assumed that $\delta < \frac{1}{96t}$. By rearranging the terms in the inequality we see that 
\begin{align*}
\alpha_\mathbf{j} &\geq \frac{1}{16 \eps^2 \users^2} \left(\frac{\users}{3}- \alpha_\mathbf{j} -8\delta \users^2\right)\\
& \geq \frac{1}{16 \eps^2 \users^2} \frac{\users^2}{12^2} = \frac{1}{2304\eps^2}.
\end{align*}
Thus, $\alpha_\mathbf{j} \geq \min\left\{\frac{t}{6}, \frac{1}{2304\eps^2}\right\}$.

We now incorporate the randomness over the choice of $\mathbf{j}$. Let $E$ be the event that the vector of target indices has no duplicates. Since $\datadim \geq c \users^2$ for a sufficiently large constant $c$, by \Cref{lem:birthdayp} event $E$ occurs with probability at least $\frac{1}{2}$. Therefore,
\begin{align*}
    4\alpha &\geq \ex{\mathbf{j}}{\alpha_\mathbf{j}}\\
            & \geq \ex{\mathbf{j}}{\alpha_\mathbf{j}\mid E}\pr{\mathbf{j}}{E} \\
            &\geq \frac{1}{2} \ex{\mathbf{j}}{\alpha_\mathbf{j}\mid E}\pr{\mathbf{j}}{E}.
\end{align*}
For each $\mathbf{j}$ without duplicates we have a lower bound on $\alpha_\mathbf{j}$, so
\[
\alpha \geq \frac{1}{8}\min\left\{\frac{t}{6}, \frac{1}{2304\eps^2}\right\}.
\]
\end{proof}

\begin{customthm}{\ref*{thm:meanestmetalb}}[Restated]
    \thmMeanEstMetaLB
\end{customthm}
\begin{proof}
    We consider that the metalearning algorithm $\alg_{\text{meta}}$ takes as input 1 sample $(x^{(i)}, j^{(i)})$ per person $i \in [\users]$. Then the personalization algorithm $\alg_{\text{pers}}$ gets as input the output of $\alg_{\text{meta}}$ and the sample of the $(\users+1)$-th person $(x^{(\users+1)}, j^{(\users+1)})$ and outputs an estimate $\hat{p}_{j^{(\users+1)}}$ of the mean of coordinate $j^{(\users+1)}$. By the definition of $\distfam_{\datadim, \users+1}$, the index of person $\users+1$ is deterministically  $j_{t+1}$. Hence, we will write $\hat{p}_{j_{\users+1}}$ instead of $\hat{p}_{j^{(\users+1)}}$ for simplicity.
    
    To prove this theorem we construct a hard metadistribution $\metadist$ where we draw a vector of means $\mathbf{p} \in [-1,+1]^\datadim$ and a vector of indices $\mathbf{j} \in [\datadim]^{\users+1}$ both uniformly at random. Let the error of $\alg$ be 
    \begin{align*}
        4\alpha &\geq \ex{\substack{\alg, \\ (\dist_1, \ldots, \dist_{\users+1}) \sim \metadist,\\ x^{(1)}, \ldots, x^{(\users+1)}}}{(\hat{p}_{j_{\users+1}} - p_{j_{\users+1}})^2} \\
        & = \ex{\substack{\alg, \\ \mathbf{p}, j_1, \ldots, j_{\users},\\ x^{(1)}, \ldots, x^{(\users+1)}}}{\ex{j_{\users+1}}{(\hat{p}_{j_{\users+1}} - p_{j_{\users+1}})^2}} \\
        & = \ex{\substack{\alg, \\ \mathbf{p}, j_1, \ldots, j_{\users},\\ x^{(1)}, \ldots, x^{(\users+1)}}}{\frac{1}{\datadim}\sum_{j \in [\datadim]}(\hat{p}_{j} - p_{j})^2}.
    \end{align*}
    
    We consider a tracing attack that uses the following test statistics for $i \in [\users]$
    \begin{align*}
        T_i& \defeq \sum_{j \in [\datadim]}(\hat{p}_j - p_j)(x_j^{(i)} - p_j) \text{ and}\\
        T_i'& \defeq \sum_{j \in [\datadim]}({\hat{p}_j}' - p_j)(x^{(i)}_j - p_j),
    \end{align*}
    where ${\hat{p}_j}'$ denotes the output of algorithm $\alg$ for $j_{\users+1} = j$ when the input of person $i$ to $\alg_{\text{meta}}$ has been replaced with a fresh draw from $\dist_i$. For $i = \users+1$ we construct only test statistic
    \[
      T_{\users+1}\defeq \sum_{j \in [\datadim]}(\hat{p}_j - p_j)(x_j^{(t+1)} - p_j).
    \]
    Since $\alg_{\text{meta}}$ is $(\eps, \delta)$-DP with respect to $i$'s dataset for every $i \in [\users]$, $\hat{p}_j$ is $(\eps,\delta)$-DP with respect to the same dataset for every $j \in [\datadim]$. Therefore,
    \[
    \ex{}{T_i} \leq \ex{}{T_i'} + 2\eps \sqrt{\var{}{T_i'}} + 2\delta \|T_i'\|_{\infty}.
    \]

    We will now analyze each term of the right hand side of the inequality. We see that  
    \begin{align*}
        \|T_i'\|_{\infty} \leq 4 \datadim
    \end{align*}
    because $T_i'$ is the sum of $\datadim$ entries of value at most $4$.
    Next, since $\hat{p}_j'$ is independent of $x^{(i)}_j$ conditioned on $\mathbf{p}$, we get that 
    \begin{align*}
        \ex{}{T_i'} & = \ex{\alg, \mathbf{p}}{ \sum_{j \in [\datadim]} \ex{\substack{\alg,j_1, \ldots, j_\users\\ x^{(1)}, \ldots, x^{(\users+1)}, {x^{(i)}}'}}{({\hat{p}_j}'-p_j)}\ex{x^{(i)}}{(x_j^{(i)}-p_j)}} \\
        &= 0
    \end{align*}
    Finally, by the same observation the crosse terms in the variance of $T_i'$ cancel out and we obtain that 
    \begin{align*}
        \var{}{T_i'} &= \ex{}{(T_i')^2} \\
        &= \ex{}{\sum_{j\in [\datadim]} (\hat{p}_j' -p_j)^2(x_j^{(i)}-p_j)^2}\\
        &\leq 4 \ex{}{\sum_{j\in [\datadim]} (\hat{p}_j' -p_j)^2} \leq 16\datadim\alpha
    \end{align*}
    Combining the inequalities above we conclude that 
    \[
     \ex{}{T_i} \leq 8 \eps  \sqrt{\alpha} \sqrt{\datadim}+ 8\delta \datadim.
    \]
    For the $i = \users+1$, we have that
    \begin{align*}
        \ex{}{T_{\users+1}} & = \ex{}{\sum_{j \in [\datadim]} (\hat{p}_j -p_j) (x_j^{(\users+1)} - p_j)}\\
        & \leq 2\ex{}{\sum_{j \in [\datadim]} |\hat{p}_j -p_j| }\\
        & \leq 2\ex{}{ \sqrt{\datadim}\sqrt{\sum_{j \in [\datadim]} (\hat{p}_j -p_j)^2} }\\
        & \leq 2\sqrt{\datadim}\sqrt{\ex{}{ \sum_{j \in [\datadim]} (\hat{p}_j -p_j)^2 }}\\
        & \leq 4 \datadim\sqrt{\alpha}.
    \end{align*}

    Therefore, by summing up the test statistict $T_i$
    \[
    \ex{}{\sum_{i \in [\users+1]} T_i} \leq 8 \eps \sqrt{\alpha} \users \sqrt{\datadim} + 8\delta \datadim \users + 4 \sqrt{\alpha} \datadim.
    \]

    The next step is to show that accuracy implies a lower bound for the test statistics in terms of error $\alpha$.
    We apply \Cref{lem:fp-lemma} to every coordinate $j \in [\datadim]$ of the estimate

    \begin{align*}
        \ex{}{\sum_{i \in [\users+1]}T_i}& = \ex{}{\sum_{i \in [\users+1]}\sum_{j \in [\datadim]}(\hat{p}_j - p_j)(x_j^{(i)} - p_j)}\\
        & = \sum_{j \in [\datadim]}\ex{}{\sum_{i \in [\users+1]}(\hat{p}_j - p_j)(x_j^{(i)} - p_j)}\\
        & = \sum_{j \in [\datadim]} \ex{j_1, \ldots, j_{\users}}{\ex{}{(\hat{p}_j - p_j)\sum_{i \in [\users+1]}(x_j^{(i)} - p_j)}}\\
        & \geq \sum_{j \in [\datadim]} \left( \frac{1}{3} - \ex{}{(\hat{p}_j - p_j)^2}\right) = \frac{\datadim}{3} - 4\datadim\alpha.
    \end{align*}

    Combining the bounds on $\ex{}{\sum_{i \in [\users+1]}T_i}$, we have the following inequality
    \[
    \frac{\datadim}{3} - 4\datadim\alpha \leq 8 \eps \sqrt{\alpha} \users\sqrt{\datadim} + 8\delta \datadim \users + 4 \sqrt{\alpha} \datadim.
    \]
    If we rearrange the terms we get that 
    \[
    \alpha \geq \frac{1}{16 \eps^2\users^2 \datadim} \left(\frac{\datadim}{3} - 4\sqrt{\alpha} \datadim - 4\alpha\datadim - 8\delta\datadim \users\right)^2.
    \]
    If $\alpha \leq \frac{1}{4 \cdot 144}$, then since $\delta \leq \frac{1}{96\users}$
    \begin{align*}
        \alpha &\geq \frac{1}{8^2 \eps^2 \users^2}\left(\frac{\datadim}{3} - \frac{\datadim}{6} - \frac{\datadim}{144} -\frac{\datadim}{12} \right)^2\\
        & \geq \frac{11^2 \datadim}{8^2 \cdot 144^2 \eps^2 \users^2}.
    \end{align*}
    Therefore, 
    \[\alpha \geq \min\left\{ \frac{1}{4 \cdot 144}, \frac{11\datadim}{8^2 \cdot 144^2\eps^2 \users^2}\right\}\]
\end{proof}
\begin{customcor}{\ref*{thm:meanestbblb}}[Restated]
    \thmMeanEstBBLB
\end{customcor}
\begin{proof}
    By \Cref{thm:metatomulti} if we have a billboard algorithm for  multitask learning with error $\alpha$, then we have metalearning algorithm with error $e^\eps\alpha +\delta$. Then, by \Cref{thm:meanestmetalb} 
    \[
    e^\eps\alpha +\delta \geq \min\left\{ \frac{1}{4 \cdot 144}, \frac{11\datadim}{8^2 \cdot 144^2\eps^2 \users^2}\right\}.
    \]
    We have made the assumption that $\datadim \geq \frac{16 \cdot 144}{11} \users$. 
    
    If $\frac{11\datadim}{ 8^2 \cdot 144^2 \eps^2 \users^2} <\frac{1}{4\cdot144}$, then for $\users \leq \frac{11\datadim}{8 \cdot 144}$, $\delta <\frac{1}{8\cdot 144\users}$ and $\eps \leq 1$, we have \[
    \alpha \geq \frac{11\datadim}{e^\eps 2 \cdot8^2 \cdot 144^2 \eps^2 \users^2 } \geq \frac{11\datadim}{6 \cdot 8^2 \cdot 144^2 \eps^2 \users^2} .\]
    
    If $\frac{11\datadim}{ 8^2 \cdot 144^2 \eps^2 \users^2} \geq \frac{1}{4 \cdot144}$, since $\users \geq 1$, $\delta \leq \frac{1}{8 \cdot144 \users}$ and $\eps \in [0,1]$, we have that $\alpha \geq \frac{1}{3\cdot8\cdot  144}$. Therefore,
    \[
    \alpha \geq \Omega\left(\min \left\{1, \frac{\datadim}{\eps^2 \users^2}\right\}\right)
    \] 
    
\end{proof}
\section{Proofs from \Cref{sec:classification}}

\subsection{Proof of the Fingerprinting Lemma}
\label{sec:fp-proof}

We break the proof of our fingerprinting lemma into smaller lemmas following the proof in \cite{BunSU17}.
\begin{lem}
\label{lem:our-fp-lemma-a}
    Let $f:\{\pm1\}^\users \to \R$. Define $g:[\pm1] \to \R$ by
    \(
    g(p) = \ex{x_1,\ldots,x_\users \sim p}{f(x)}.
    \)
    Then,
    \[
    \ex{x_1,\ldots,x_\users}{f(x) \sum_{i \in [\users]} (x_i-p)]} = g'(p) (1-p)^2.
    \]
\end{lem}
\begin{proof}
    We can write 
    \[
    g(p) = \ex{x_1,\ldots, x_\users}{f(x)} = \sum_{x_1, \ldots, x_\users \in \{\pm 1\}^\users} f(x) \prod_{i \in [\users]} \frac{1+x_ip}{2},
    \]
    where $x_1, \ldots, x_\users \in \{\pm 1\}^\users$ are sampled independently with mean $p$.
    We can now compute its derivative
    \begin{align*}
        g'(p) & = \sum_{x_1, \ldots, x_\users \in \{\pm 1\}^\users}f(x) \frac{d}{dp}\prod_{i \in [\users]} \frac{1+x_ip}{2}\\
        & = \sum_{x_1, \ldots, x_\users \in \{\pm 1\}^\users}f(x) \sum_{i \in [\users]} \frac{x_i-p}{1-p^2} \prod_{j \in [\users] \frac{1+x_j p}{2}}\\
        & = \ex{x_1,\ldots, x_\users}{f(x) \sum_{i \in [\users]} \frac{x_i - p}{1-p^2}}
    \end{align*}
\end{proof}
\begin{lem}
\label{lem:our-fp-lemma-b}
    Let $g: [\pm 1] \to \R$ be a polynomial and $\alpha \in (0,1]$. If $p \in [-\alpha, +\alpha]$ is drawn uniformly at random, then
    \[
     \ex{p }{r(\alpha,p) g'(p) (1-p^2)} = 2 \ex{p }{g(p)p},
    \]
    where 
     \[
     r(\alpha,p) = \begin{cases}
         \frac{\alpha^2-p^2}{1-p^2}, & \text{if }\alpha^2 \neq 1,\\
         1, &\text{otherwise}.\\
     \end{cases}
     \]
\end{lem}
\begin{proof}
    If $\alpha^2 \neq 1$,
    \begin{align*}
        \ex{p }{ r(\alpha,p) g'(p) (1-p^2) }
        & = \frac{1}{2\alpha}\int_{-\alpha}^\alpha r(\alpha,p) g'(p)(1-p^2) dp \\
        &= \frac{1}{2\alpha}\int_{-\alpha}^\alpha (\alpha^2-p^2) g'(p) dp \\
        & = \frac{1}{2\alpha} \int_{-\alpha}^\alpha \frac{d}{dp}\left[ (\alpha^2-p^2) g(p)\right] - g(p)(-2p)dp\\
        & = \frac{1}{2\alpha} \int_{-\alpha}^{\alpha} 2g(p)pdp = 2 \ex{p }{ g(p)p}.
    \end{align*}
    Similarly, we can show that when $\alpha^2 = 1$
    \[
     \ex{p }{  g'(p) (1-p^2) } = 2 \ex{p}{ g(p)p}.
    \]
\end{proof}
\begin{proof}[Proof of Lemma~\ref{lem:our-fp-lemma}]
   Applying Lemmas~\ref{lem:our-fp-lemma-a} and \ref{lem:our-fp-lemma-b} we have 
   \begin{align*}
       \ex{p , x_1, \ldots, x_\users}{r(\alpha,p) f(x) \sum_{i \in [\users]} (x_i -p) +2(|p| - f(x)p)} 
       &= \ex{p }{r(\alpha,p) \ex{x_1,\ldots,x_\users }{f(x) \sum_{i \in [\users]} (x_i -p)} } + \ex{p , x_1, \ldots, x_\users }{2(|p|-f(x)p) }\\
       &= \ex{p , x_1,\ldots, x_\users }{ 2 f(x) p}+ \ex{p , x_1, \ldots, x_\users }{2(|p|-f(x)p) }\\
       & = \ex{p }{2|p|} = \alpha
   \end{align*}
\end{proof}
We can show that the fingerprinting lemma in \cite{BunUV14} follows from our fingerprinting lemma for $\alpha = 1$.

 \begin{lem} [Fingerprinting Lemma \cite{BunUV14}] 
     Let $f: \{\pm 1\}^\users \to [\pm1]$. If $p \in [-1,+1]$ is sampled uniformly at random and $x_1,\ldots,x_\users \in \{\pm 1\}^\users$ are sampled independently with mean $p$, then
     \[
     \ex{p , x}{(f(x)-p) \sum_{i \in [\users]}  (x_i-p) +(f(x)-p)^2} \geq \frac{1}{3}.
     \]
     \label{lem:fp-lemma}
 \end{lem}
 \begin{proof}
      We have that 
      \begin{align*}
         &\ex{p , x_1, \ldots , x_\users }{(f(x)-p) \sum_{i \in [\users]}  (x_i-p) +(f(x)-p)^2}\\
         & = \ex{p , x_1, \ldots, x_\users }{f(x) \sum_{i \in [\users]}  (x_i-p) +f(x)^2-2f(x)p} + \ex{p }{p^2 -p \ex{x_1, \ldots, x_\users}{\sum_{i \in [\users]} (x_i-p)}}\\
         & \geq \ex{p, x_1, \ldots, x_\users}{f(x) \sum_{i \in [\users]}  (x_i-p) -2f(x)p} + \ex{p}{2|p|} -\frac{2}{3}\\
         & = \ex{p, x_1, \ldots, x_\users}{f(x) \sum_{i \in [\users]}  (x_i-p) +2(|p| -f(x)p)}  -\frac{2}{3}\\
         & \geq 1 -\frac{2}{3} =\frac{1}{3}.
      \end{align*}
      The last inequality follows from Lemma~\ref{lem:our-fp-lemma}.
 \end{proof}

\subsection{Lower Bounds}

In this section we provide the proofs for the lower bounds of the error of indexed classification in the private personalization frameworks we consider. We first prove lower bounds for indexed sign estimation and then prove the corresponding lower bounds for indexed classification by using the reduction of \Cref{thm:redesttoclass}.
\label{sec:indClass-lb-proofs}

\newcommand{\thmSignEstJDPLB}{Fix parameters $\alpha \in (0,1/16)$, $\users \in \N$, $\eps >0$, $\delta \in (0, \frac{1}{2\users})$, and let $\ell_{\text{sign}}$ be the loss function. Let $\alg$ be an algorithm that multitask learns $\distfam_{\text{est},\datadim, \users}$ with error $\alpha$, for loss function $\ell_\text{sign}$, with $\users$ tasks and $2$ samples per task and satisfies $(\eps, \delta)$-JDP.
If $d\ge ct^2$ for a sufficiently large constant $c$, then $\alpha \ge \Omega(\min\{\frac{\sqrt{\users}}{ \eps},1\})$.}

\begin{lem}[Indexed sign estimation JDP lower bound]
    \thmSignEstJDPLB
    \label{thm:signestjdplb}
\end{lem}
\begin{proof}
Fix parameters $\alpha \in (0,1/16)$, $\users \in \N$, $\eps >0$, $\delta \in (0, \frac{1}{2\users})$.
Let $\alg$ be a JDP algorithm that gets as input the $\datadim$-dimensional samples of $\users$ tasks/people. Each person $i \in [\users]$, gives samples $(x^{(i,1)}, j^{(i,1)}), (x^{(i,2)}, j^{(i,2)})$ to the algorithm. We start by constructing a hard distribution over the family of distributions $\distfam_{\text{est},\datadim,\users}$ that the samples are drawn from. 
We draw a vector $\mathbf{p} \in [-\lambda, \lambda]^\datadim$ for $\lambda = 16\alpha$ and a vector $\mathbf{j} \in [\datadim]^\users$, both uniformly. The $k$-th datapoint $x^{(i,k)}$ of person $i$ is drawn from the product distribution with mean $\textbf{p}$ and the index $j^{(i,k)}$ is deterministically $j_i$. 

$\alg$ returns an estimate $\hat{s}_{j_i}$ of the sign of $p_{j_i}$ to each person $i \in [\users]$. If the expected error of $\alg$ is
\[
  \ex{\mathbf{p}, \mathbf{j}}{\frac{1}{\users}\sum_{i \in [\users]}\ex{}{\ind\{\text{sign}(p_{j_i})\neq \hat{s}_{j_i}\}|p_{j_i}|}} \leq \alpha.
\]
our goal is to prove a lower bound on $\alpha$.

 We will first analyze the case where $\mathbf{j}$ has no duplicated indices. So for now we assume that we have a fixed $\mathbf{j}$ with no repeats. For person $i \in [\users]$ and sample $k \in \{1,2\}$ we define two test statistics 
\begin{align*}
    T_{i,k} & \defeq \sum_{\ell \in [\users] \setminus \{i\}} \frac{\lambda^2-p_{j_\ell}^2}{1-p_{j_\ell}^2}\hat{s}_{j_\ell}(x^{(i,k)}_{j_\ell} - p_{j_\ell})\\
    T_{i,k}' & \defeq \sum_{\ell \in [\users] \setminus \{i\}} \frac{\lambda^2-p_{j_\ell}^2}{1-p_{j_\ell}^2}\hat{s}_{j_\ell}^{(i,k)}(x^{(i,k)}_{j_\ell} - p_{j_\ell})
\end{align*}
where $\hat{s}_{j_\ell}^{(i,k)}$ denotes the output of algorithm $\alg$ to person $\ell$ when the input $\left(x^{(i,k)}, j^{(i,k)}\right)$ if person $i$ has been replaced with a fresh draw from $P_i$. Since the distribution over the indices is deterministic, $j^{(i,k)}$ is always $ j_i$. 

Now, we will use the privacy guarantee of JDP to provide an upper bound for $\ex{}{\sum_{i \in [\users]}(T_{i,1}+T_{i,2})}$. Since $\alg$ is $(\eps, \delta)$-JDP, we have that $\{\hat{s}_{j_\ell}\}_{\ell \in [\users]\setminus i}$ is $(\eps, \delta)$-DP with respect to $i$'s dataset. As a result, the following inequality holds 
\begin{align*}
    \ex{}{T_{i,k}} \leq \ex{}{{T_{i,k}}'} + 2 \eps \sqrt{\var{}{{T_{i,k}}'}} + 2 \delta \|T_{i,k}'\|_{\infty}.
\end{align*}

$T_i'$ is a sum of $\users-1$ values that are at most $2\lambda^2$, so the bound $\|T_i'\|_{\infty} \leq 2\lambda^2(\users-1) \leq 2\lambda^2\users$ holds. For fixed $\mathbf{p}$, and $\ell$, $\hat{s}_{j_\ell}^{(i,k)}$ is independent of $x_{j_\ell}^{(i,k)}$ and $\ex{}{x_{j_\ell}^{(i,k)}} = p_{j_\ell}$. Thus, for any $\ell \in [\users] \setminus \{i\}$
\begin{align*}
 \ex{}{ \frac{\lambda^2-p_{j_\ell}^2}{1-p_{j_\ell}^2}\hat{s}_{j_\ell}^{(i,k)}(x^{(i)}_{j_\ell} - p_{j_\ell})} 
= \ex{\mathbf{p}}{ \frac{\lambda^2-p_{j_\ell}^2}{1-p_{j_\ell}^2}\ex{}{\hat{s}_{j_\ell}^{(i,k)}}\ex{x_{j_\ell}^{(i,k)}}{(x_{j_\ell}^{(i,k)} - p_{j_\ell})}} = 0.
\end{align*}
This means that $\ex{}{T_{i,k}'} = 0$. We apply the same observation, that $\ex{}{x_{j_\ell}^{(i,k)}} = p_{j_\ell}$, and that every coordinate is independent to show that the cross terms in the variance of $T_{i,j}'$ cancel out, leaving us with 
\[
\var{}{T_{i,k}'} = \ex{}{(T_i')^2} = \ex{}{\sum_{\ell \in [\users]\setminus{i}}\frac{(\lambda^2-p_{j_\ell}^2)^2}{(1-p_{j_\ell}^2)^2}(\hat{s}_{j_\ell}^{(i,k)})^2(x^{(i)}_{j_\ell} - p_{j_\ell})^2}. 
\]
Since $(x^{(i)}_{j_\ell} - p_{j_\ell})^2$ is at most $4$ and $(\hat{s}_{j_\ell}^{(i,k)})^2=1$, we have the upper bound 
\begin{align*}
    \var{}{T_{i,k}'}& \leq 4(\users-1) \lambda^4\leq 4\users \lambda^4.
\end{align*}
Plugging these into inequality above and summing up over all the $\users$ people and $k \in \{1,2\}$, we obtain that 
\begin{align*}
    \ex{}{\sum_{i\in[\users]}\sum_{k \in \{1,2\}}T_{i,k}} \leq 4\eps \lambda^2 \sqrt{\users} + 4\delta \lambda^2 \users.
\end{align*}

The next step is to show that accuracy implies a lower bound for $\ex{}{\sum_{i \in [\users]}(T_{i,1} + T_{i,2})}$. We notice that in the sum of $T_i$ we can rearrange the terms.
\begin{align*}
    \sum_{i \in [\users]} (T_{i,1}+T_{i,2}) &= \sum_{i \in [\users]} \sum_{\ell \in [\users]\setminus\{i\}}\left(\frac{\lambda^2-p_{j_\ell}^2}{1-p_{j_\ell}^2}\hat{s}_{j_\ell}(x^{(i,1)}_{j_\ell} - p_{j_\ell})+\frac{\lambda^2-p_{j_\ell}^2}{1-p_{j_\ell}^2}\hat{s}_{j_\ell}(x^{(i,2)}_{j_\ell} - p_{j_\ell})\right)\\
    & = \sum_{\ell \in [\users]} \sum_{i \in [\users] \setminus \{\ell\}} \left(\frac{\lambda^2-p_{j_\ell}^2}{1-p_{j_\ell}^2}\hat{s}_{j_\ell}(x^{(i,1)}_{j_\ell} - p_{j_\ell})+\frac{\lambda^2-p_{j_\ell}^2}{1-p_{j_\ell}^2}\hat{s}_{j_\ell}(x^{(i,2)}_{j_\ell} - p_{j_\ell})\right)\\
    & = \sum_{\ell \in [\users]} \frac{\lambda^2-p_{j_\ell}^2}{1-p_{j_\ell}^2}\hat{s}_{j_\ell}\sum_{i \in [\users] \setminus \{\ell\}} \sum_{k \in \{1,2\}}(x^{(i,k)}_{j_\ell} - p_{j_\ell}).\\
\end{align*}
We can now apply \Cref{lem:our-fp-lemma} to each coordinate $j_\ell$.
\begin{align*}
    \ex{}{\sum_{i\in [\users]}(T_{i,1} + T_{i,2})} & = \sum_{\ell \in [\users]} \ex{}{\frac{\lambda^2-p_{j_\ell}^2}{1-p_{j_\ell}^2}\hat{s}_{j_\ell}\sum_{i \in [\users] \setminus \{\ell\}} \sum_{k \in \{1,2\}}(x^{(i,k)}_{j_\ell} - p_{j_\ell})}\\
    & \geq \sum_{\ell \in [\users]}\left( \lambda - \ex{}{2(|p_{j_\ell}| - \hat{s}_{j_\ell} p_{j_\ell})}\right)\\
    & = \lambda \users - \ex{}{\sum_{\ell \in [\users]}2(|p_{j_\ell}| - \hat{s}_{j_\ell} p_{j_\ell})}\,.
\end{align*}
For vector of indices $\mathbf{j}$ let $\alpha_{\mathbf{j}}$ be the error 
\[
\alpha_{\mathbf{j}} = \ex{}{ \sum_{\ell \in [\users]}\ind\{\text{sign}(p_{j_\ell}) \neq \hat{s}_{j_\ell}\}|p_{j_\ell}|}  = \frac{1}{4}\ex{}{\sum_{\ell \in [\users]}2(|p_{j_\ell}| - \hat{s}_{j_\ell} p_{j_\ell})}.
\]
Combining the two inequalities for $ \ex{}{\sum_{i\in [\users]}(T_{i,1}+T_{i,2})}$ we have shown that when $\mathbf{j}$ has no duplicates 
\[
\lambda \users - 4\alpha_{\mathbf{j}} \leq \ex{}{\sum_{i\in [\users]}(T_{i,1} + T_{i,2})}  \leq 4\eps \lambda^2 \sqrt{\users} + 4\delta \lambda^2\users.
\]
By rearranging the terms we get that 
\[
\alpha_{\mathbf{j}} \geq \frac{1}{4} (\lambda \users -4 \eps \lambda^2 \sqrt{\users}-4\delta \lambda^2 \users).
\]

We now incorporate the randomness over the choice of $\mathbf{j}$. Let $E$ be the event that the set of target indices has no duplicates. Since $\datadim \geq c \users^2 $ for a sufficiently large constant $c$, by \Cref{lem:birthdayp} event $E$ occurs with probability at least $\frac{1}{2}$. Therefore,
\begin{align*}
\alpha = \ex{\mathbf{p}, \mathbf{j}}{\frac{1}{\users}\sum_{i \in [\users]}\ex{}{\ind\{\text{sign}(p_{j_i})\neq \hat{s}_{j_i}\}|p_{j_i}|}} &=
    \frac{1}{\users}\ex{\mathbf{j}}{\alpha_\mathbf{j}}\geq \frac{1}{\users} \ex{\mathbf{j}}{\alpha_\mathbf{j}\mid E}\pr{\mathbf{j}}{E} \geq \frac{1}{2\users} \ex{\mathbf{j}}{\alpha_\mathbf{j}\mid E}.
\end{align*}
For each $\mathbf{j}$ without duplicates we have a lower bound on $\alpha_\mathbf{j}$, so substituting this bound for $\ex{\mathbf{j}}{\alpha_\mathbf{j}\mid E}$, we get
\begin{align*}
\alpha \geq \frac{1}{8\users}(\lambda \users - 4\eps \lambda^2 \sqrt{\users} - 4\delta \lambda^2 \users)
\end{align*}
Since $\lambda = 16\alpha$, $\delta < \frac{1}{2\users}$ and $\users \geq 1$, we get that
\begin{align*}
    \alpha \geq \frac{1}{32} \frac{1}{\frac{4\eps}{\sqrt{\users}}+4\delta} \geq \frac{1}{32} \frac{1}{\frac{4\eps}{\sqrt{\users}}+2} \geq \frac{1}{32} \min\left\{\frac{\sqrt{\users}}{ 4\eps}, \frac{1}{2}\right\} = \min\left\{\frac{\sqrt{\users}}{4\cdot32\eps}, \frac{1}{64}\right\}.
\end{align*}
\end{proof}

\newcommand{\thmSignEstMetaLB}{ Fix parameters $\alpha \in (0,\frac{1}{8})$, $\users \in \N$, $\eps \in (0, 1]$, $\delta \in (0, \frac{1}{2\users})$, and error function $\ell_{\text{est}}$. 
Let $\alg = (\alg_{\text{meta}}, \alg_{\text{pers}})$ be a pair of algorithms that metalearn a metadistribution $\metadist$ over $\distfam_{\text{est},\datadim, \users+1}$ with error $\alpha$ using $\users$ training tasks, $2$ samples per training task and a test task with $2$ personalization samples and $\alg_{\text{meta}}$ satisfies $(\eps, \delta)$-DP. Then, $\alpha \ge \Omega(\min\{1, \frac{\sqrt{\datadim}}{\eps \users}\})$.}

\begin{lem}[Indexed sign estimation metalearning lower bound]
    \thmSignEstMetaLB
    \label{thm:signestmetalb}
\end{lem}

\begin{proof}
    Fix parameters $\alpha \in (0,\frac{1}{8})$, $\users \in \N$, $\eps \in (0, 1]$, $\delta \in (0, \frac{1}{2\users})$, and error function $\ell_{\text{est}}$.
    The metalearning algorithm $\alg_{\text{meta}}$ takes as input 2 samples $(x^{(i,1)}, j^{(i,1)}), (x^{(i,2)},j^{(i,2)})$ per person $i \in [\users]$. The personalization algorithm $\alg_{\text{pers}}$ gets as input the output of $\alg_{\text{meta}}$ and the two samples of the $(\users+1)$-th person $(x^{(\users+1,1)}, j^{(\users+1,1)}),(x^{(\users+1,2)}, j^{(\users+1,2)})$. It then outputs an estimate $\hat{s}_{j^{(\users+1,1)}}$ of the sign of the mean of coordinate $j^{(\users+1,1)}$. By the definition of $\distfam_{\text{est},\datadim, \users+1}$, the index of person $\users+1$ in both samples is deterministically  $j_{t+1}$. Hence, we will write $\hat{s}_{j_{\users+1}}$ instead of $\hat{s}_{j^{(\users+1,1)}}$ for simplicity.
    
    To prove this theorem we construct a hard metadistribution $\metadist$ where we draw a vector of means $\mathbf{p} \in [-\lambda,+\lambda]^\datadim$, for $\lambda = 8\alpha \in (0,1)$ and a vector of indices $\mathbf{j} \in [\datadim]^{\users+1}$ both uniformly at random. Let the error of $\alg$ be 
    \begin{align*}
        \alpha & \geq \ex{\substack{\alg, \\ (\dist_1, \ldots, \dist_{\users+1}) \sim \metadist,\\ x^{(1,1)}, x^{(1,2)}, \ldots, x^{(\users+1,1)},x^{(\users+1,2)}}}{\ind\{\text{sign}(p_{j_{\users+1}}) \neq \hat{s}_{j_{\users+1}}\}|p_{j_{\users+1}}|} \\
        & = \ex{\substack{\alg, \\ \mathbf{p}, j_1, \ldots, j_{\users},\\ x^{(1,1)}, x^{(1,2)}, \ldots, x^{(\users+1,1)},x^{(\users+1,2)}}}{\ex{j_{\users+1}}{\ind\{\text{sign}(p_{j_{\users+1}}) \neq \hat{s}_{j_{\users+1}}\}|p_{j_{\users+1}}|}} \\
        & = \ex{\substack{\alg, \\ \mathbf{p}, j_1, \ldots, j_{\users},\\ x^{(1,1)}, x^{(1,2)}, \ldots, x^{(\users+1,1)},x^{(\users+1,2)}}}{\frac{1}{\datadim}\sum_{j \in [\datadim]}\ind\{\text{sign}(p_{j}) \neq \hat{s}_{j}\}|p_{j}|}.
    \end{align*}

    We construct a tracing attack that uses the following test statistics for $i \in [\users]$, $k \in \{1,2\}$
    \begin{align*}
        T_{i,k}& \defeq \sum_{j \in [\datadim]}\frac{\lambda^2-p_j^2}{1-p_j^2}\hat{s}_j(x_j^{(i,k)} - p_j) \text{ and}\\
        T_{i,k}'& \defeq \sum_{j \in [\datadim]}\frac{\lambda^2-p_j^2}{1-p_j^2}{\hat{s}_j^{(i,k)}}(x_j^{(i,k)} - p_j),
    \end{align*}
    where $\hat{s}_j^{(i,k)}$ denotes the output of algorithm $\alg_{\text{pers}}$ for $j_{\users+1} = j$ when the input $(x^{(i,k)}, j^{(i,k)})$ of person $i$ to $\alg_{\text{meta}}$ has been replaced with a fresh draw from $\dist_i$. For $i = \users+1$ we construct only test statistics 
    \[
      T_{\users+1,k}\defeq \sum_{j \in [\datadim]}\frac{\lambda^2-p_j^2}{1-p_j^2}\hat{s}_j(x_j^{(i,k)} - p_j),
    \]
    for $k \in \{1,2\}$.
    Since $\alg_{\text{meta}}$ is $(\eps, \delta)$-DP with respect to $i$'s dataset for every $i \in [\users]$, $\hat{s}_j$ is $(\eps,\delta)$-DP with respect to the same dataset for every $j \in [\datadim]$. Therefore, for $k \in\{1,2\}$
    \[
    \ex{}{T_{i,k}} \leq \ex{}{T_{i,k}'} + 2\eps \sqrt{\var{}{T_{i,k}'}} + 2\delta \|T_{i,k}'\|_{\infty}.
    \]

    We will now analyze each term of the right hand side of the inequality. We see that  
    \begin{align*}
        \|T_{i,k}'\|_{\infty} &\leq \sum_{j \in [\datadim]}\lambda^2 \left| \frac{1 - \frac{p_j^2}{\lambda^2}}{1-p_j^2} \hat{s}_j^{(i,k)} \left(x_j^{(i,k)} - p_j\right)\right|\\
        & \leq \sum_{i \in [\datadim]} 2\lambda^2 = 2 \lambda^2 \datadim,
    \end{align*}
    because $1 - \frac{p_j^2}{\lambda^2}\leq{1-p_j^2}$.
    Next, since $\hat{s}_j^{(i,k)}$ is independent of $x^{(i)}_j$ conditioned on $\mathbf{p}$, we get that 
    \begin{align*}
        \ex{}{T_i'} & = \ex{\alg, \mathbf{p}}{ \sum_{j \in [\datadim]} \frac{\lambda^2-p_j^2}{1-p_j^2}\ex{\substack{j_1, \ldots, j_\users\\ x^{(1,1)}, \ldots, x^{(\users+1,2)}, {x^{(i,k)}}'}}{\hat{s}_j^{(i,k)}}\ex{x^{(i,k)}}{(x_j^{(i,k)}-p_j)}} \\
        &= 0
    \end{align*}
    Finally, by the same observation the cross terms in the variance of $T_{i,k}'$ cancel out and we obtain that 
    \begin{align*}
        \var{}{T_{i,k}'} &= \ex{}{(T_{i,k}')^2} \\
        &= \ex{}{\sum_{j\in [\datadim]} \frac{(\lambda^2 - p_j^2)^2}{(1-p_j^2)^2}(\hat{s}_j^{(i,k)})^2(x_j^{(i,k)}-p_j)^2}\\
        &\leq 4 \ex{}{\sum_{j\in [\datadim]} \frac{(\lambda^2 - p_j^2)^2}{(1-p_j^2)^2}} = 4\datadim\lambda^4
    \end{align*}
    Combining the inequalities above we conclude that 
    \[
     \ex{}{T_{i,k}} \leq 4 \eps  \lambda^2 \sqrt{\datadim}+ 4\delta \datadim \lambda^2.
    \]
    For the $i = \users+1$, we have that
    \begin{align*}
        \ex{}{T_{\users+1,k}} & = \ex{}{\sum_{j \in [\datadim]} \frac{\lambda^2-p_j^2}{1-p_j^2} \hat{s}_j (x_j^{(\users+1,k)} - p_j)}\\
        & \leq 2 \datadim\lambda^2.
    \end{align*}

    Therefore, by summing up the test statistics $T_{i,k}$
    \[
    \ex{}{\sum_{i \in [\users+1]} (T_{i,1}+T_{i,2})} \leq 8 \eps \lambda^2 \users \sqrt{\datadim} + 8\delta \datadim \users \lambda^2 + 4 \lambda^2 \datadim.
    \]

    The next step is to show that accuracy implies a lower bound for the test statistics in terms of error $\lambda$.
    We apply \Cref{lem:our-fp-lemma} to every coordinate $j \in [\datadim]$ of the estimate

    \begin{align*}
        \ex{}{\sum_{i \in [\users+1]}\sum_{k \in \{1,2\}}T_{i,k}}& = \ex{}{\sum_{i \in [\users+1]}\sum_{k \in \{1,2\}}\sum_{j \in [\datadim]}\frac{\lambda^2-p_j^2}{1-p_j^2}\hat{s}_j (x_j^{(i,k)} - p_j)}\\
        & = \sum_{j \in [\datadim]}\ex{}{\sum_{i \in [\users+1]}\sum_{k \in \{1,2\}}\frac{\lambda^2-p_j^2}{1-p_j^2}\hat{s}_j(x_j^{(i,k)} - p_j)}\\
        & \geq \sum_{j \in [\datadim]} \left( \lambda - \ex{}{4 \ind\{\text{sign}(p_j) \neq \hat{s}_j\}|p_j|}\right) = \datadim \lambda - 4\datadim\alpha.
    \end{align*}

    Combining the bounds on $\ex{}{\sum_{i \in [\users+1]}\sum_{k \in \{1,2\}}T_{i,k}}$, we have the following inequality
    \[
    \datadim \lambda - 4\datadim\alpha \leq  8 \eps \lambda^2 \users \sqrt{\datadim} + 8\delta \datadim \users \lambda^2 + 4 \lambda^2 \datadim.
    \]
    By rearranging the terms and replacing $\lambda$ with $8\alpha$ we have that
    \[
        \alpha \geq \frac{1}{16} \frac{1}{\frac{8\eps \users}{\sqrt{\datadim}} + 8\delta \users + 4}.
    \]
    Since $\delta < \frac{1}{2\users}$, 
    \[
        \alpha \geq \frac{1}{16} \frac{1}{\frac{8\eps \users}{\sqrt{\datadim}} + 8}.
     \]
    Finally, if $\frac{8\eps\users}{\sqrt{\datadim}} \leq 8$, then
    $
    \alpha \geq \frac{\sqrt{\datadim}}{16^2 \eps \users}$.
    Otherwise, $\alpha \geq \frac{1}{16^2}$.
    Therefore, we have shown that \[
    \alpha \geq \min\left\{\frac{\sqrt{\datadim}}{16^2 \eps \users},\frac{1}{16^2}\right\}
    \]
\end{proof}

\newcommand{\thmSignEstBBLB}{ Fix parameters $\users \in \N$, $\eps \in (0,1]$, $\delta\in (0, \frac{1}{32^2t})$, $d\ge \frac{\eps^2 t}{4}$. Let $\ell_\text{sign}$ be the loss function . 
Let $\alg$ be a billboard algorithm that multitask learns $\distfam_{\text{est}, \datadim, \users}$ with error $\alpha$ with $\users$ tasks and $2$ samples per task and satisfies $(\eps, \delta)$-DP.
Then, $\alpha \ge \Omega(\min\{\frac{\sqrt{d}}{\epsilon t},1\})$.}

\begin{lem}[Indexed sign estimation billboard lower bound]
    \thmSignEstBBLB
    \label{thm:signestbblb}
\end{lem}

\begin{proof}
By \Cref{thm:metatomulti} if we have a billboard algorithm for  multitask learning with error $\alpha$, then we have a metalearning algorithm with error $e^\eps\alpha +\delta$ for metadistributions over $\dist_{\text{est}, \datadim, \users+1}$. Then, by \Cref{thm:signestmetalb}
\[
e^\eps\alpha +\delta \geq \min\left\{\frac{\sqrt{\datadim}}{16^2 \eps \users},\frac{1}{16^2}\right\}.
\]
If $\frac{\sqrt{\datadim}}{ 16^2 \eps \users} <\frac{1}{16^2}$, then since $\delta <\frac{1}{32^2 \users}$, $\datadim \geq \frac{\eps^2 \users}{4}$, $\users \geq 1$, and $\eps \leq 1$, we have that
\begin{align}
\alpha & \geq \frac{1}{e^\eps} \left (\frac{\sqrt{\datadim}}{16^2 \eps \users} -\delta\right)\\
&\geq \frac{\sqrt{\datadim}}{e^\eps 2\cdot 16^2 \eps \users } \geq \frac{\sqrt{\datadim}}{6\cdot 16^2 \eps \users } .
\end{align}
If $\frac{\sqrt{\datadim}}{ 16^2 \eps \users} \geq \frac{1}{16^2}$, since $\delta < \frac{1}{32^2 \users}$ we get that 
\[
\alpha \geq \frac{3}{e^\eps 4\cdot 16^2}>\frac{1}{4\cdot 16^2}.
\]
Therefore,
\[
\alpha \geq \Omega\left(\min \left\{1, \frac{\sqrt{\datadim}}{\eps \users}\right\}\right)
\] 
    
\end{proof}

\begin{customthm}{\ref*{thm:classjdplb}}[Restated]
    \thmClassJDPLB
\end{customthm}
\begin{proof}
    If there exists an $(\eps,\delta)$-JDP algorithm $\alg$ that multitask learns $\distfam_{\text{class}, \datadim, \users}$ with $\users$ tasks and $1$ sample per task to error $\alpha < \min\left\{\frac{\sqrt{\users}}{4 \cdot 32 \eps}, \frac{1}{64}\right\}$, then by the reduction in \Cref{thm:redclasstoest} we get that there exists a sign estimation algorithm that multitask learns $\distfam_{\text{est}, \datadim, \users}$ with $\users$ tasks and $2$ samples per task and has error $\alpha < \min\left\{\frac{\sqrt{\users}}{4 \cdot 32 \eps}, \frac{1}{64}\right\}$, for loss function $\ell_\text{sign}$. The algorithm we get from the reduction is also $(\eps, \delta)$-JDP due to post-processing. This contradicts the statement of \Cref{thm:signestjdplb}. Therefore, we get that $\alpha$ must be at least $\min\left\{\frac{\sqrt{\users}}{4\cdot32\eps}, \frac{1}{64}\right\}$.
\end{proof}

\begin{customthm}{\ref*{thm:classmetalb}}[Restated]
    \thmClassMetaLB
\end{customthm}
\begin{proof}

We notice that a metadistribution $\metadist_{\text{class}}$ over $\distfam_{\text{class}, \datadim, \users+1}$ where we draw a vector of means $\mathbf{p} \in [-\lambda, \lambda]^{\datadim}$, for $\lambda \in (0,1)$ and a vector of indices $\mathbf{j} \in [\datadim]^{\users+1}$ uniformly at random corresponds to metadistribution $\metadist_{\text{est}}$ over $\distfam_{\text{est}, \datadim, \users+1}$ where we draw $\mathbf{p}$ and $\mathbf{j}$ in the same way.

We will show that we can reduce metalearning for metadistribution $\metadist_{\text{est}}$ for indexed sign estimation with $\users$ training tasks, $1$ sample per training task and a test task with $2$ personalization samples to metalearning for the corresponding distribution $\metadist_{\text{class}}$ for indexed classification using $\users$ training tasks with $1$ sample per training task and a test task with $1$ personalization samples by following the steps of the proof of \Cref{thm:redesttoclass} with some small changes. 

In indexed sign estimation, let $(x^{(i)}, j^{(i)})$ be the sample of individual $i$ that is drawn from $\dist_\text{class}^{(i)}$, for $(\dist_\text{class}^{(1)}, \ldots, \dist_\text{class}^{(\users+1)})$ drawn from $\metadist_{\text{class}}$.

We can transform every sample $(x^{(i)}, j^{(i)})$, for $i \in [\users]$ to a sample for indexed classification by setting $w^{(i)} \gets x^{(i)}$, drawing a $y^{(i)} \in \{\pm 1\}$ uniformly and setting 
\begin{align*}
    \tilde{x}_\ell^{(i)} = \begin{cases}
        w_\ell^{(i)}, &\text{ if } \ell \neq j_i\\
         w_\ell^{(i)}y, &\text{ if } \ell = j_i\\
    \end{cases}
\end{align*}

The distribution of $(\tilde{x}^{(i)}, j^{(i)}, y^{(i)})$ is $\dist_\text{class}^{(i)}$ and the distribution of $(\dist_\text{class}^{(1)}, \ldots, \dist_\text{class}^{(\users+1)})$ is $\metadist_{\text{class}}$. Every person $i \in [\users]$ sends their new sample, $(\tilde{x}^{(i)}, j^{(i)}, y^{(i)})$, to the metalearning algorithm for indexed classification. 

Person $\users+1$ transform their two datapoints from $\{(x^{(\users+1,1)}, j^{(\users+1,1)})),(x^{(\users+1,2)}, j^{(\users+1,2)}))\}$ to $\{(\tilde{x}^{(\users+1,1)}, j^{(\users+1,1)}, y^{(\users+1,1)}),$ $(\tilde{x}^{(\users+1,2)}, j^{(\users+1,2)}, y^{(\users+1,2)})\}$ using the same procedure as the people with the training tasks. They then run the personalization part of that algorithm with sample $(\tilde{x}^{(\users+1,1)}, j^{(\users+1,1)}, y^{(\users+1,1)})$ and get a $\hat{f}(x_{j_{\users+1}})$  with error 
\begin{align*}
    \ex{}{\pr{(x,j,y) \sim \dist_\text{class}^{(\users+1)}}{\hat{f}(x_{j_{\users+1}}) \neq y} - \min_f \left\{\pr{(x,j,y) \sim \dist_\text{class}^{(\users+1)}}{f(x_{j_{\users+1}}) \neq y}\right\}} 
\end{align*}
where the expectation is taken over the $\users$-tuple of classification distributions, the samples of the $\users$ training tasks, the randomness of the algorithm and the first sample of individual $\users+1$. Then, person $\users+1$ can get an estimate of the sign by postprocessing $\hat{f}$ using their second transformed sample $(\tilde{x}^{(\users+1,2)}, j^{(\users+1,2)}, y^{(\users+1,2)})$:
\[
\hat{s}_{j_\users} = \hat{f}(\tilde{x}^{(\users+1,2)}_{j_{\users+1}})\tilde{x}^{(\users+1,2)}_{j_{\users+1}}.
\]

Following the same calculations as in the proof of \Cref{thm:redesttoclass} we get that 
\begin{align*}
    \ex{}{\ind\{\text{sign}(p_{j_{\users+1}}) \neq \hat{s}_{j_{\users+1}}\}|p_{j_{\users+1}}|} = \ex{}{\pr{}{\hat{f}(x_{j_{\users+1}}) \neq y} - \min_f \{\pr{}{f(x_{j_{\users+1}}) \neq y}\} }\leq  \alpha.
\end{align*}
where the expectation of the LHS of the equation is taken over the $\users$-tuple of indexed sign estimation distributions, the initial samples of the $\users$ training tasks, the randomness of the algorithm we described and the two samples of individual $\users+1$.
By \Cref{thm:signestmetalb} we have that $\alpha \geq \Omega(\min\{1, \frac{\sqrt{\datadim}}{\eps \users}\})$.

\end{proof}

\begin{customthm}{\ref*{thm:classbblb}}[Restated]
    \thmClassBBLB
\end{customthm}
\begin{proof}
    If there exists an $(\eps, \delta)$-DP billboard algorithm $\alg$ that multitask $\distfam_{\text{class}, \datadim, \users}$ with $\users$ tasks and $1$ sample per task with error $\alpha < \min\left\{\frac{\sqrt{\datadim}}{6 \cdot 16^2 \eps\users }, \frac{1}{32^2}\right\}$, then by the reduction in \Cref{thm:redclasstoest} we get that there exists a sign estimation algorithm that multitask learns $\distfam_{\text{est}, \datadim, \users}$ with $\users$ tasks and $2$ samples per task and has error $\alpha < \min\left\{\frac{\sqrt{\datadim}}{6 \cdot 16^2 \eps\users }, \frac{1}{32^2}\right\}$. The billboard algorithm we get from the reduction is also $(\eps, \delta)$-DP due to post-processing. This contradicts the statement of \Cref{thm:signestbblb}. Therefore, we get that $\alpha$ must be at least $\min\left\{\frac{\sqrt{\datadim}}{6 \cdot 16^2 \eps\users }, \frac{1}{32^2}\right\}$.
\end{proof}

\end{document}